\documentclass{article}

\usepackage{microtype}
\usepackage{graphicx}
\usepackage{subfigure}
\usepackage{booktabs} % for professional tables
\usepackage{def}
\usepackage{hyperref}

\usepackage[accepted]{icml2021}

\icmltitlerunning{PAC-Learning for Strategic Classification}

\begin{document}
%%%%% NEW MATH DEFINITIONS %%%%%

% Mark sections of captions for referring to divisions of figures
\newcommand{\figleft}{{\em (Left)}}
\newcommand{\figcenter}{{\em (Center)}}
\newcommand{\figright}{{\em (Right)}}
\newcommand{\figtop}{{\em (Top)}}
\newcommand{\figbottom}{{\em (Bottom)}}
\newcommand{\captiona}{{\em (a)}}
\newcommand{\captionb}{{\em (b)}}
\newcommand{\captionc}{{\em (c)}}
\newcommand{\captiond}{{\em (d)}}

% Highlight a newly defined term
\newcommand{\newterm}[1]{{\bf #1}}

% Figure reference, lower-case.
\def\figref#1{figure~\ref{#1}}
% Figure reference, capital. For start of sentence
\def\Figref#1{Figure~\ref{#1}}
\def\twofigref#1#2{figures \ref{#1} and \ref{#2}}
\def\quadfigref#1#2#3#4{figures \ref{#1}, \ref{#2}, \ref{#3} and \ref{#4}}
% Section reference, lower-case.
\def\secref#1{section~\ref{#1}}
% Section reference, capital.
\def\Secref#1{Section~\ref{#1}}
% Reference to two sections.
\def\twosecrefs#1#2{sections \ref{#1} and \ref{#2}}
% Reference to three sections.
\def\secrefs#1#2#3{sections \ref{#1}, \ref{#2} and \ref{#3}}
% Reference to an equation, lower-case.
% \def\eqref#1{equation~\ref{#1}}
% Reference to an equation, upper case
% \def\Eqref#1{Equation~\ref{#1}}
% A raw reference to an equation---avoid using if possible
\def\plaineqref#1{\ref{#1}}
% Reference to a chapter, lower-case.
\def\chapref#1{chapter~\ref{#1}}
% Reference to an equation, upper case.
\def\Chapref#1{Chapter~\ref{#1}}
% Reference to a range of chapters
\def\rangechapref#1#2{chapters\ref{#1}--\ref{#2}}
% Reference to an algorithm, lower-case.
\def\algref#1{algorithm~\ref{#1}}
% Reference to an algorithm, upper case.
\def\Algref#1{Algorithm~\ref{#1}}
\def\twoalgref#1#2{algorithms \ref{#1} and \ref{#2}}
\def\Twoalgref#1#2{Algorithms \ref{#1} and \ref{#2}}
% Reference to a part, lower case
\def\partref#1{part~\ref{#1}}
% Reference to a part, upper case
\def\Partref#1{Part~\ref{#1}}
\def\twopartref#1#2{parts \ref{#1} and \ref{#2}}

\def\ceil#1{\lceil #1 \rceil}
\def\floor#1{\lfloor #1 \rfloor}
\def\1{\bm{1}}
\newcommand{\train}{\mathcal{D}}
\newcommand{\valid}{\mathcal{D_{\mathrm{valid}}}}
\newcommand{\test}{\mathcal{D_{\mathrm{test}}}}

\def\eps{{\epsilon}}

% Random variables
\def\reta{{\textnormal{$\eta$}}}
\def\ra{{\textnormal{a}}}
\def\rb{{\textnormal{b}}}
\def\rc{{\textnormal{c}}}
\def\rd{{\textnormal{d}}}
\def\re{{\textnormal{e}}}
\def\rf{{\textnormal{f}}}
\def\rg{{\textnormal{g}}}
\def\rh{{\textnormal{h}}}
\def\ri{{\textnormal{i}}}
\def\rj{{\textnormal{j}}}
\def\rk{{\textnormal{k}}}
\def\rl{{\textnormal{l}}}
% rm is already a command, just don't name any random variables m
\def\rn{{\textnormal{n}}}
\def\ro{{\textnormal{o}}}
\def\rp{{\textnormal{p}}}
\def\rq{{\textnormal{q}}}
\def\rr{{\textnormal{r}}}
\def\rs{{\textnormal{s}}}
\def\rt{{\textnormal{t}}}
\def\ru{{\textnormal{u}}}
\def\rv{{\textnormal{v}}}
\def\rw{{\textnormal{w}}}
\def\rx{{\textnormal{x}}}
\def\ry{{\textnormal{y}}}
\def\rz{{\textnormal{z}}}

% Random vectors
\def\rvepsilon{{\mathbf{\epsilon}}}
\def\rvtheta{{\mathbf{\theta}}}
\def\rva{{\mathbf{a}}}
\def\rvb{{\mathbf{b}}}
\def\rvc{{\mathbf{c}}}
\def\rvd{{\mathbf{d}}}
\def\rve{{\mathbf{e}}}
\def\rvf{{\mathbf{f}}}
\def\rvg{{\mathbf{g}}}
\def\rvh{{\mathbf{h}}}
\def\rvu{{\mathbf{i}}}
\def\rvj{{\mathbf{j}}}
\def\rvk{{\mathbf{k}}}
\def\rvl{{\mathbf{l}}}
\def\rvm{{\mathbf{m}}}
\def\rvn{{\mathbf{n}}}
\def\rvo{{\mathbf{o}}}
\def\rvp{{\mathbf{p}}}
\def\rvq{{\mathbf{q}}}
\def\rvr{{\mathbf{r}}}
\def\rvs{{\mathbf{s}}}
\def\rvt{{\mathbf{t}}}
\def\rvu{{\mathbf{u}}}
\def\rvv{{\mathbf{v}}}
\def\rvw{{\mathbf{w}}}
\def\rvx{{\mathbf{x}}}
\def\rvy{{\mathbf{y}}}
\def\rvz{{\mathbf{z}}}

% Elements of random vectors
\def\erva{{\textnormal{a}}}
\def\ervb{{\textnormal{b}}}
\def\ervc{{\textnormal{c}}}
\def\ervd{{\textnormal{d}}}
\def\erve{{\textnormal{e}}}
\def\ervf{{\textnormal{f}}}
\def\ervg{{\textnormal{g}}}
\def\ervh{{\textnormal{h}}}
\def\ervi{{\textnormal{i}}}
\def\ervj{{\textnormal{j}}}
\def\ervk{{\textnormal{k}}}
\def\ervl{{\textnormal{l}}}
\def\ervm{{\textnormal{m}}}
\def\ervn{{\textnormal{n}}}
\def\ervo{{\textnormal{o}}}
\def\ervp{{\textnormal{p}}}
\def\ervq{{\textnormal{q}}}
\def\ervr{{\textnormal{r}}}
\def\ervs{{\textnormal{s}}}
\def\ervt{{\textnormal{t}}}
\def\ervu{{\textnormal{u}}}
\def\ervv{{\textnormal{v}}}
\def\ervw{{\textnormal{w}}}
\def\ervx{{\textnormal{x}}}
\def\ervy{{\textnormal{y}}}
\def\ervz{{\textnormal{z}}}

% Random matrices
\def\rmA{{\mathbf{A}}}
\def\rmB{{\mathbf{B}}}
\def\rmC{{\mathbf{C}}}
\def\rmD{{\mathbf{D}}}
\def\rmE{{\mathbf{E}}}
\def\rmF{{\mathbf{F}}}
\def\rmG{{\mathbf{G}}}
\def\rmH{{\mathbf{H}}}
\def\rmI{{\mathbf{I}}}
\def\rmJ{{\mathbf{J}}}
\def\rmK{{\mathbf{K}}}
\def\rmL{{\mathbf{L}}}
\def\rmM{{\mathbf{M}}}
\def\rmN{{\mathbf{N}}}
\def\rmO{{\mathbf{O}}}
\def\rmP{{\mathbf{P}}}
\def\rmQ{{\mathbf{Q}}}
\def\rmR{{\mathbf{R}}}
\def\rmS{{\mathbf{S}}}
\def\rmT{{\mathbf{T}}}
\def\rmU{{\mathbf{U}}}
\def\rmV{{\mathbf{V}}}
\def\rmW{{\mathbf{W}}}
\def\rmX{{\mathbf{X}}}
\def\rmY{{\mathbf{Y}}}
\def\rmZ{{\mathbf{Z}}}

% Elements of random matrices
\def\ermA{{\textnormal{A}}}
\def\ermB{{\textnormal{B}}}
\def\ermC{{\textnormal{C}}}
\def\ermD{{\textnormal{D}}}
\def\ermE{{\textnormal{E}}}
\def\ermF{{\textnormal{F}}}
\def\ermG{{\textnormal{G}}}
\def\ermH{{\textnormal{H}}}
\def\ermI{{\textnormal{I}}}
\def\ermJ{{\textnormal{J}}}
\def\ermK{{\textnormal{K}}}
\def\ermL{{\textnormal{L}}}
\def\ermM{{\textnormal{M}}}
\def\ermN{{\textnormal{N}}}
\def\ermO{{\textnormal{O}}}
\def\ermP{{\textnormal{P}}}
\def\ermQ{{\textnormal{Q}}}
\def\ermR{{\textnormal{R}}}
\def\ermS{{\textnormal{S}}}
\def\ermT{{\textnormal{T}}}
\def\ermU{{\textnormal{U}}}
\def\ermV{{\textnormal{V}}}
\def\ermW{{\textnormal{W}}}
\def\ermX{{\textnormal{X}}}
\def\ermY{{\textnormal{Y}}}
\def\ermZ{{\textnormal{Z}}}

% Vectors
\def\vzero{{\bm{0}}}
\def\vone{{\bm{1}}}
\def\vmu{{\bm{\mu}}}
\def\vtheta{{\bm{\theta}}}
\def\va{{\bm{a}}}
\def\vb{{\bm{b}}}
\def\vc{{\bm{c}}}
\def\vd{{\bm{d}}}
\def\ve{{\bm{e}}}
\def\vf{{\bm{f}}}
\def\vg{{\bm{g}}}
\def\vh{{\bm{h}}}
\def\vi{{\bm{i}}}
\def\vj{{\bm{j}}}
\def\vk{{\bm{k}}}
\def\vl{{\bm{l}}}
\def\vm{{\bm{m}}}
\def\vn{{\bm{n}}}
\def\vo{{\bm{o}}}
\def\vp{{\bm{p}}}
\def\vq{{\bm{q}}}
\def\vr{{\bm{r}}}
\def\vs{{\bm{s}}}
\def\vt{{\bm{t}}}
\def\vu{{\bm{u}}}
\def\vv{{\bm{v}}}
\def\vw{{\bm{w}}}
\def\vx{{\bm{x}}}
\def\vy{{\bm{y}}}
\def\vz{{\bm{z}}}

% Elements of vectors
\def\evalpha{{\alpha}}
\def\evbeta{{\beta}}
\def\evepsilon{{\epsilon}}
\def\evlambda{{\lambda}}
\def\evomega{{\omega}}
\def\evmu{{\mu}}
\def\evpsi{{\psi}}
\def\evsigma{{\sigma}}
\def\evtheta{{\theta}}
\def\eva{{a}}
\def\evb{{b}}
\def\evc{{c}}
\def\evd{{d}}
\def\eve{{e}}
\def\evf{{f}}
\def\evg{{g}}
\def\evh{{h}}
\def\evi{{i}}
\def\evj{{j}}
\def\evk{{k}}
\def\evl{{l}}
\def\evm{{m}}
\def\evn{{n}}
\def\evo{{o}}
\def\evp{{p}}
\def\evq{{q}}
\def\evr{{r}}
\def\evs{{s}}
\def\evt{{t}}
\def\evu{{u}}
\def\evv{{v}}
\def\evw{{w}}
\def\evx{{x}}
\def\evy{{y}}
\def\evz{{z}}

% Matrix
\def\mA{{\bm{A}}}
\def\mB{{\bm{B}}}
\def\mC{{\bm{C}}}
\def\mD{{\bm{D}}}
\def\mE{{\bm{E}}}
\def\mF{{\bm{F}}}
\def\mG{{\bm{G}}}
\def\mH{{\bm{H}}}
\def\mI{{\bm{I}}}
\def\mJ{{\bm{J}}}
\def\mK{{\bm{K}}}
\def\mL{{\bm{L}}}
\def\mM{{\bm{M}}}
\def\mN{{\bm{N}}}
\def\mO{{\bm{O}}}
\def\mP{{\bm{P}}}
\def\mQ{{\bm{Q}}}
\def\mR{{\bm{R}}}
\def\mS{{\bm{S}}}
\def\mT{{\bm{T}}}
\def\mU{{\bm{U}}}
\def\mV{{\bm{V}}}
\def\mW{{\bm{W}}}
\def\mX{{\bm{X}}}
\def\mY{{\bm{Y}}}
\def\mZ{{\bm{Z}}}
\def\mBeta{{\bm{\beta}}}
\def\mPhi{{\bm{\Phi}}}
\def\mLambda{{\bm{\Lambda}}}
\def\mSigma{{\bm{\Sigma}}}

% Tensor
\newcommand{\tens}[1]{\bm{\mathsfit{#1}}}
\def\tA{{\tens{A}}}
\def\tB{{\tens{B}}}
\def\tC{{\tens{C}}}
\def\tD{{\tens{D}}}
\def\tE{{\tens{E}}}
\def\tF{{\tens{F}}}
\def\tG{{\tens{G}}}
\def\tH{{\tens{H}}}
\def\tI{{\tens{I}}}
\def\tJ{{\tens{J}}}
\def\tK{{\tens{K}}}
\def\tL{{\tens{L}}}
\def\tM{{\tens{M}}}
\def\tN{{\tens{N}}}
\def\tO{{\tens{O}}}
\def\tP{{\tens{P}}}
\def\tQ{{\tens{Q}}}
\def\tR{{\tens{R}}}
\def\tS{{\tens{S}}}
\def\tT{{\tens{T}}}
\def\tU{{\tens{U}}}
\def\tV{{\tens{V}}}
\def\tW{{\tens{W}}}
\def\tX{{\tens{X}}}
\def\tY{{\tens{Y}}}
\def\tZ{{\tens{Z}}}

% Graph
\def\gA{{\mathcal{A}}}
\def\gB{{\mathcal{B}}}
\def\gC{{\mathcal{C}}}
\def\gD{{\mathcal{D}}}
\def\gE{{\mathcal{E}}}
\def\gF{{\mathcal{F}}}
\def\gG{{\mathcal{G}}}
\def\gH{{\mathcal{H}}}
\def\gI{{\mathcal{I}}}
\def\gJ{{\mathcal{J}}}
\def\gK{{\mathcal{K}}}
\def\gL{{\mathcal{L}}}
\def\gM{{\mathcal{M}}}
\def\gN{{\mathcal{N}}}
\def\gO{{\mathcal{O}}}
\def\gP{{\mathcal{P}}}
\def\gQ{{\mathcal{Q}}}
\def\gR{{\mathcal{R}}}
\def\gS{{\mathcal{S}}}
\def\gT{{\mathcal{T}}}
\def\gU{{\mathcal{U}}}
\def\gV{{\mathcal{V}}}
\def\gW{{\mathcal{W}}}
\def\gX{{\mathcal{X}}}
\def\gY{{\mathcal{Y}}}
\def\gZ{{\mathcal{Z}}}

% Sets
\def\sA{{\mathbb{A}}}
\def\sB{{\mathbb{B}}}
\def\sC{{\mathbb{C}}}
\def\sD{{\mathbb{D}}}
% Don't use a set called E, because this would be the same as our symbol
% for expectation.
\def\sF{{\mathbb{F}}}
\def\sG{{\mathbb{G}}}
\def\sH{{\mathbb{H}}}
\def\sI{{\mathbb{I}}}
\def\sJ{{\mathbb{J}}}
\def\sK{{\mathbb{K}}}
\def\sL{{\mathbb{L}}}
\def\sM{{\mathbb{M}}}
\def\sN{{\mathbb{N}}}
\def\sO{{\mathbb{O}}}
\def\sP{{\mathbb{P}}}
\def\sQ{{\mathbb{Q}}}
\def\sR{{\mathbb{R}}}
\def\sS{{\mathbb{S}}}
\def\sT{{\mathbb{T}}}
\def\sU{{\mathbb{U}}}
\def\sV{{\mathbb{V}}}
\def\sW{{\mathbb{W}}}
\def\sX{{\mathbb{X}}}
\def\sY{{\mathbb{Y}}}
\def\sZ{{\mathbb{Z}}}

% Entries of a matrix
\def\emLambda{{\Lambda}}
\def\emA{{A}}
\def\emB{{B}}
\def\emC{{C}}
\def\emD{{D}}
\def\emE{{E}}
\def\emF{{F}}
\def\emG{{G}}
\def\emH{{H}}
\def\emI{{I}}
\def\emJ{{J}}
\def\emK{{K}}
\def\emL{{L}}
\def\emM{{M}}
\def\emN{{N}}
\def\emO{{O}}
\def\emP{{P}}
\def\emQ{{Q}}
\def\emR{{R}}
\def\emS{{S}}
\def\emT{{T}}
\def\emU{{U}}
\def\emV{{V}}
\def\emW{{W}}
\def\emX{{X}}
\def\emY{{Y}}
\def\emZ{{Z}}
\def\emSigma{{\Sigma}}

% entries of a tensor
% Same font as tensor, without \bm wrapper
\newcommand{\etens}[1]{\mathsfit{#1}}
\def\etLambda{{\etens{\Lambda}}}
\def\etA{{\etens{A}}}
\def\etB{{\etens{B}}}
\def\etC{{\etens{C}}}
\def\etD{{\etens{D}}}
\def\etE{{\etens{E}}}
\def\etF{{\etens{F}}}
\def\etG{{\etens{G}}}
\def\etH{{\etens{H}}}
\def\etI{{\etens{I}}}
\def\etJ{{\etens{J}}}
\def\etK{{\etens{K}}}
\def\etL{{\etens{L}}}
\def\etM{{\etens{M}}}
\def\etN{{\etens{N}}}
\def\etO{{\etens{O}}}
\def\etP{{\etens{P}}}
\def\etQ{{\etens{Q}}}
\def\etR{{\etens{R}}}
\def\etS{{\etens{S}}}
\def\etT{{\etens{T}}}
\def\etU{{\etens{U}}}
\def\etV{{\etens{V}}}
\def\etW{{\etens{W}}}
\def\etX{{\etens{X}}}
\def\etY{{\etens{Y}}}
\def\etZ{{\etens{Z}}}

% The true underlying data generating distribution
\newcommand{\pdata}{p_{\rm{data}}}
% The empirical distribution defined by the training set
\newcommand{\ptrain}{\hat{p}_{\rm{data}}}
\newcommand{\Ptrain}{\hat{P}_{\rm{data}}}
% The model distribution
\newcommand{\pmodel}{p_{\rm{model}}}
\newcommand{\Pmodel}{P_{\rm{model}}}
\newcommand{\ptildemodel}{\tilde{p}_{\rm{model}}}
% Stochastic autoencoder distributions
\newcommand{\pencode}{p_{\rm{encoder}}}
\newcommand{\pdecode}{p_{\rm{decoder}}}
\newcommand{\precons}{p_{\rm{reconstruct}}}

\newcommand{\laplace}{\mathrm{Laplace}} % Laplace distribution

\newcommand{\E}{\mathbb{E}}
\newcommand{\Ls}{\mathcal{L}}
\newcommand{\R}{\mathbb{R}}
\newcommand{\emp}{\tilde{p}}
\newcommand{\lr}{\alpha}
\newcommand{\reg}{\lambda}
\newcommand{\rect}{\mathrm{rectifier}}
\newcommand{\softmax}{\mathrm{softmax}}
\newcommand{\sigmoid}{\sigma}
\newcommand{\softplus}{\zeta}
\newcommand{\KL}{D_{\mathrm{KL}}}
\newcommand{\standarderror}{\mathrm{SE}}
% Wolfram Mathworld says $L^2$ is for function spaces and $\ell^2$ is for vectors
% But then they seem to use $L^2$ for vectors throughout the site, and so does
% wikipedia.
\newcommand{\normlzero}{L^0}
\newcommand{\normlone}{L^1}
\newcommand{\normltwo}{L^2}
\newcommand{\normlp}{L^p}
\newcommand{\normmax}{L^\infty}

\newcommand{\parents}{Pa} % See usage in notation.tex. Chosen to match Daphne's book.

\let\ab\allowbreak

 \newcommand{\todo}[1]{{\color{red}\noindent[ToDos: #1]}}

\newcommand{\linit}{\ell_{init}}
\newcommand{\fan}[1]{\textcolor{blue}{#1}}
\newcommand{\anil}[1]{\textcolor{red}{#1}}

\newenvironment{proofof}[1]{\begin{proof}[Proof of #1]}{\end{proof}}
\newenvironment{proofsketch}{\begin{proof}[Proof Sketch]}{\end{proof}}
\newenvironment{proofsketchof}[1]{\begin{proof}[Proof Sketch of #1]}{\end{proof}}

%regular version
\def\pr{\qopname\relax n{Pr}}
\def\ex{\qopname\relax n{E}}
\def\min{\qopname\relax n{min}}
\def\max2{\qopname\relax n{max2}}
\def\max{\qopname\relax n{max}}
\def\argmin{\qopname\relax n{argmin}}
\def\argmax{\qopname\relax n{argmax}}
\def\avg{\qopname\relax n{avg}}

%bold version
\def\Pr{\qopname\relax n{\mathbf{Pr}}}
\def\Ex{\qopname\relax n{\mathbf{E}}}

\newcommand{\RR}{\mathbb{R}}
\newcommand{\NN}{\mathbb{N}}
\newcommand{\ZZ}{\mathbb{Z}}
\newcommand{\QQ}{\mathbb{Q}}
\newcommand{\II}{\mathbb{I}}

\def\A{\mathcal{A}}
\def\B{\mathcal{B}}
\def\C{\mathcal{C}}
\def\D{\mathcal{D}}
\def\E{\mathcal{E}}
\def\F{\mathcal{F}}
\def\G{\mathcal{G}}
\def\H{\mathcal{H}}
\def\I{\mathcal{I}}
\def\J{\mathcal{J}}
\def\L{\mathcal{L}}
\def\M{\mathcal{M}}
\def\P{\mathcal{P}}
\def\R{\mathcal{R}}
\def\S{\mathcal{S}}
\def\O{\mathcal{O}}
\def\T{\mathcal{T}}
\def\V{\mathcal{V}}
\def\X{\mathcal{X}}
\def\Y{\mathcal{Y}}

\def\eps{\epsilon}

\def \cD {\mathcal{D}}
\def \cG {\mathcal{G}}
\def \cH {\mathcal{H}}
\def \cR {R}
\def \cX {\mathcal{X}}
\def \cY {\mathcal{Y}}

%%% Vectors 
\def\x{\bm{x}} 
\def\XX{\textbf{X}} 
\def\y{\bm{y}} 
\def\w{\bm{w}} 
\def\c{\bm{c}} 
\def\e{\bm{e}} 
\def\r{\bm{r}} 
\def\v{\bm{v}} 
\def\z{\bm{z}} 
\def\p{\bm{p}} 
\def\q{\bm{q}} 
\def\u{\bm{u}} 
\def\v{\bm{v}}

%LP environment stuff
\newcommand{\mini}[1]{\mbox{minimize} & {#1} &\\}
\newcommand{\maxi}[1]{\mbox{maximize} & {#1 } & \\}
\newcommand{\maxis}[1]{\mbox{max} & {#1 } & \\}
\newcommand{\minis}[1]{\mbox{min} & {#1 } & \\}
\newcommand{\find}[1]{\mbox{find} & {#1 } & \\}
\newcommand{\stt}{\mbox{subject to} }
\newcommand{\sts}{\mbox{s.t.} }
\newcommand{\con}[1]{&#1 & \\}
\newcommand{\qcon}[2]{&#1, & \mbox{for } #2.  \\}
\newenvironment{lp}{\begin{equation}  \begin{array}{lll}}{\end{array}\end{equation} }
\newenvironment{lp*}{\begin{equation*}  \begin{array}{lll}}{\end{array}\end{equation*}}

%misc
\newcommand{\probdet}{\textsc{StraC}}
\newcommand{\prob}{\textsc{StraC}$\langle \cH, R, c  \rangle$} 
\newcommand{\probL}{\textsc{StraC}$\langle \cH_d, R, c  \rangle$} 
\newcommand{\probrand}{\textsc{RandSC}}

\newcommand{\blue}[1]{\textcolor{blue}{#1}}

\twocolumn[
\icmltitle{PAC-Learning for Strategic Classification}



\icmlsetsymbol{equal}{*}

\begin{icmlauthorlist}
\icmlauthor{Ravi Sundaram}{equal,K}
\icmlauthor{Anil Vullikanti}{equal,CS,B}
\icmlauthor{Haifeng Xu}{equal,CS}
\icmlauthor{Fan Yao}{equal,CS}
\end{icmlauthorlist}

% \author[1]{Ravi Sundaram\thanks{r.sundaram@northeastern.edu}}
% \author[2,3]{Anil Vullikanti \thanks{vsakumar@virginia.edu}}
% \author[2]{Haifeng Xu\thanks{hx4ad@virginia.edu}}
% \author[2]{Fan Yao\thanks{fy4bc@virginia.edu}}

\icmlaffiliation{K}{Khoury College of Computer Science, Northeastern University, Boston, MA 02115}
\icmlaffiliation{CS}{Department of Computer Science, University of Virginia, Charlottesville, VA 22904}
\icmlaffiliation{B}{Biocomplexity Institute and Initiative, University of Virginia, Charlottesville, VA 22904}

\icmlcorrespondingauthor{Haifeng Xu}{hx4ad@virginia.edu}
\icmlcorrespondingauthor{Fan Yao}{fy4bc@virginia.edu}

\icmlkeywords{Machine Learning, ICML}

\vskip 0.3in
]



\printAffiliationsAndNotice{\icmlEqualContribution} % otherwise use the standard text.

\begin{abstract}

The study of \emph{strategic} or \emph{adversarial} manipulation of testing data to fool a classifier has attracted much recent attention.  Most previous works have focused on two extreme situations where any testing data point either is completely adversarial or always equally prefers the positive label. In this paper, we   generalize both of these through a unified framework for strategic classification, and introduce the notion of \emph{strategic VC-dimension} (SVC) to capture the PAC-learnability in our general strategic setup. SVC  provably generalizes the recent concept of adversarial VC-dimension (AVC) introduced by \citet{cullina:nips18}.
We instantiate our framework for the fundamental \emph{strategic linear classification} problem. We fully characterize: (1) the \emph{statistical learnability} of linear classifiers by pinning down its SVC; (2) its \emph{computational tractability}  by pinning down the complexity of the empirical risk minimization problem.  Interestingly, the SVC  of linear classifiers is always upper bounded by its standard VC-dimension. This characterization also strictly generalizes the  AVC bound for linear classifiers in \cite{cullina:nips18}.
\end{abstract}

\section{Introduction}
\label{sec:intro}
In today's increasingly connected world, it is rare that an algorithm will act alone. When a machine  learning algorithm is used to make predictions or decisions about others who have their own preferences over the learning outcomes, it is well known (e.g., \emph{Goodhart's law}) that \emph{gaming behaviors} may arise---these have been observed in a variety of domains such as finance~\cite{gaming-loan}, online retailing \cite{hannak2014measuring},  
education \cite{hardt2016stratclass} as well as during the ongoing COVID-19 pandemic \cite{unemploy-benefits,avoid-test}. In the early months of the pandemic, simple decision rules were designed for COVID-19 testing \cite{covid-test-rule} by the CDC. However, people had different preferences for getting tested. Those with work-from-home jobs and leave benefits preferred to get tested in order to know their true health status whereas some of the people with lower income, and without leave benefits preferred not to get tested (with fears of losing their income)~\cite{avoid-test}. 
Policy makers sometimes prefer to make
classification rules confidential \cite{Citron201489WashLRev0001TS} to mitigate such gaming.
However, this is not fool-proof in general since the methods may   be reverse engineered in some cases, and transparency of ML methods is sometimes mandated by law, e.g.,~\cite{goodman2016regulations}. 
Such concerns have led to a lot of interest in designing  learning algorithms that are robust to strategic gaming behaviors of the data sources \cite{Perote2004StrategyproofEF,dekel2010incentive,hardt2016stratclass,chen:ec18,dong:ec18,cullina:nips18,awasthi2019robustness}; the present work subscribes to this literature. %Perhaps the most widely studied two classes of problems are strategic classification and strategic regression \todo{cite here}. 
% \anil{Some references like gam, une, avo look weird.}

This paper focuses on the ubiquitous binary classification problem, and we look to design classification algorithms that are robust to gaming behaviors \emph{during the test phase}. We study a strict generalization of the canonical classification setup that naturally incorporates data points' \emph{preferences} over classification outcomes (which leads to strategic behaviors as we will describe later). In particular, each data point is denoted as a tuple $(\x, y, r)$ where $\x \in \X$ and $y \in \{ -1, +1 \}$ are the \emph{feature} and \emph{label}, respectively (as in classic classification problems), and additionally, $r \in \RR$ is a real number that describes how much this data point prefers label $+1$ over $-1$. Importantly, we allow $r$ to be negative, meaning that the data point may prefer label $-1$. For instance, in the decision rules for COVID-19 testing, individuals who prefer to get tested have $r>0$, while those who prefer not to be tested have $r<0$. The magnitude $|r|$ captures how strong their preferences are. For example, in the school choice matching market, students have heterogeneous preferences over universities \cite{pathak2017really,roth2008have} and may manipulate their application materials during the admission process. Let set $R \subseteq \RR$ denote the set of all possible values that the preference value $r$ may take. Obviously, the trivial singleton set $R = \{ 0 \}$ corresponds to the classic classification setup without any preferences. Another special case of $R = \{ 1 \}$ corresponds to the situation where all data points prefer label $+1$ equally. 
%\anil{the setup might be a bit hard to follow since the notion of manipulation is not yet defined. Might be easier to follow if the basic setup is first defined, then the generalization is described} 
This is the strategic classification settings studied in several previous works \cite{hardt2016stratclass,hu:fat19,miller2019strategic}.  A third special case is   $R = \{-1, 1 \}$. This corresponds to classification under \emph{evasion} attacks \cite{biggio:ecml13,goodfellow2015explaining,li2014feature,cullina:nips18,awasthi2019robustness}, where any test data point $(\x, y)$ prefers the \emph{opposite} of its true label $y$, i.e., the ``adversarial'' assumption.  

Our model considers any \emph{general preference} set $R$. As we will show, this much richer set of preferences may sometimes make learning more difficult, both statistically and computationally, but not always. 
% \hf{not emphasizing much that issues happens during only the testing phase, also mention all papers we are citing is the same setup 
Like~\cite{hardt2016stratclass,dong:ec18,goodfellow2015explaining,cullina:nips18}, our model assumes that manipulation is only possible to the data \emph{features} and happens only during the \emph{test} phase.  Specifically, the true feature of the test data may be altered by the  strategic data point. The cost of  masking a true  feature $\x$ to appear as a different feature $\z$ is captured by a \emph{cost function} $c(\z; \x)$. Therefore, 
%the decision faced by such a strategic data point is to mask his feature $\x$ to appear as $\z$ so that $\z$ induces the preferred label of this data point, depending on the sign of his $r$. The intricacy of 
the test data point's decision needs to balance the \emph{cost} of altering feature and the \emph{reward} of inducing its preferred label captured by $r$.  As is standard in game-theoretic analysis, the test data point is assumed a rational decision maker and will choose to alter to the feature $\z$ that maximizes its quasi-linear utility $[r \cdot \II(h(\z) = 1) - c(\z; \x)]$.   This naturally gives rise to a    \emph{Stackelberg game} \cite{von2010market}. We aim to learn, from i.i.d. drawn  (unaltered) training data, the optimal classifier $h^*$ that minimizes the  0-1  classification loss,
%(equivalently, maximizes classification accuracy)
assuming  any randomly drawn test data point (from the same distribution as testing data) will respond to $h^*$ strategically.  Notably,  the data point's strategic behaviors will \emph{not} change its true label. Such behavior is  referred to as  \emph{strategic gaming}, which crucially differs from \emph{strategic improvement} studied recently \cite{kleinberg2019classifiers,miller2019strategic}.

\subsection{Overview of Our Results}
% \todo{add a table to summarize}

% \fan{\begin{enumerate}
%     \item First, we establish the PAC-learning result under the general strategic setting by introducing a novel notion of \emph{strategic VC-dimension} (SVC), which generalizes the concept of adversarial VC-dimension (AVC) introduced by \cite{cullina:nips18}. We recover the learnability result for linear classifiers from \cite{cullina:nips18} by showing SVC=AVC in that particular setting and also show SVC can be arbitrarily larger than both AVC and the standard VC dimension in general.
%     \item We also provide computational learnability result for linear classifiers. We point out that the empirical risk minimization problem is solvable when and only when the allowed strategic behavior pattern exhibits certain adversarial nature. 
%     \item Finally, we study the power and limits of randomization in strategic classification. Surprisingly, we find that randomization may strictly increase the accuracy in general, but does not help in the special case of adversarial classification under evasion attacks. (actually, why it is the case? I though we only proved Randomization does not help in any perfectly separable adversarial classification problems.)
% \end{enumerate}}

\noindent{\bf The Strategic VC-Dimension.} We  introduce the novel notion of \emph{strategic VC-dimension} SVC($\H, R, c$) which captures the learnability of any hypothesis class $\H$  when test data points' strategic behaviors are induced by cost function $c$ and preference values from any set $R \subseteq \RR$. 

$\bullet$ We prove that any strategic classification problem is agnostic PAC learnable by the empirical risk minimization paradigm with  $O \big( \epsilon^{-2}[d + \log(\frac{1}{\delta}))] \big)$ samples, where $d=$SVC($\H, R, c$). Conceptually,  this result illustrates that SVC correctly characterizes the learnability  of the hypothesis class $\H$ in our strategic setup.  

$\bullet$ Our SVC notion   generalizes the adversarial VC-dimension (AVC) introduced in \cite{cullina:nips18} for adversarial learning with evasion attacks. Formally, we prove that AVC equals precisely SVC($\H, R, c$) for $R=\{-1,1 \}$ when  data points are allowed to move within region  $\{ \z; c(\z;\x) \leq 1 \}$ in the corresponding adversarial learning setup. 
%Therefore, AVC can be viewed as SVC under restrictive  preference set $R = \{ -1, 1 \}$ for data points. 
However, for general preference set $R$, SVC can  be arbitrarily larger than both AVC and the standard VC dimension. Thus, complex strategic behaviors may indeed make the learning statistically  more difficult. Interestingly,  to our knowledge, this is the first time that adversarial learning and strategic learning are unified under the same PAC-learning framework. %  \hf{need to carefully verify this big claim. have to tongue down if not true, otherwise the paper can be killed simply due to this single sentence. } \fan{Might not be true in general, there are works discussing adversarial ML in game theoretical framework \cite{dasgupta2019survey} How about we rephrase to "the first time that adversarial learning and strategic learning are studied under the same PAC-learning framework."}

$\bullet$ We prove SVC($\H, R, c$)$\leq 2$ for any $\H$ and $R$ when $c$ is any \emph{separable} cost function (introduced by \cite{hardt2016stratclass}). Invoking our sample complexity results above, this also recovers a  main learnability result of \citeyearpar{hardt2016stratclass} and, moreover, generalizes their result to arbitrary agent preferences.

\noindent {\bf Strategic Linear Classification. }  As a case study, we  instantiate our  strategic classification framework in perhaps one of the most fundamental classification problems,  \emph{linear classification}. Here, features are in $\RR^d$ linear space. We assume the cost function $c(\z;\x)$ for any $\x$ is induced by \emph{arbitrary} seminorms of the difference $\z - \x$. We distinguish between two crucial situations: (1) \emph{instance-invariant} cost function which means the cost of altering the feature $\x$ to $\x + \Delta$ is the same for any $\x$; (2) \emph{instance-wise} cost function which allows the cost from $\x$ to $\x + \Delta$ to be different for different $\x$. 
%which turn out to significantly affect the efficiency of learning
%(note that the separable cost function considered by~\cite{hardt2016stratclass} is an instance-wise cost function). 
Our results show that the more general instance-wise costs impose significantly more difficulties in terms of both statistical learnability and computational tractability. 

$\bullet$ {\bf Statistical Learnability.} We prove that the SVC of linear classifiers is  $\infty$ for \emph{instance-wise} cost functions even when features are in $\RR^2$; in contrast, the SVC is at most $d+1$ for any \emph{instance-independent} cost functions and any $R$ when features are in $\RR^d$. This later result also strictly generalizes the AVC bound for linear classifiers proved in \cite{cullina:nips18}, and illustrates an interesting conceptual message: though SVC can be significantly larger than AVC in general,  \emph{extending from $R = \{ -1, 1\}$ (the adversarial  setting) to an arbitrary strategic preference set $R$  does not affect the statistical learnability of strategic linear classification}.
% Therefore, in linear classification, 
% \hf{let's call generally adversarial, adversarial. }

$\bullet$ {\bf Computational Tractability.} We show that the empirical risk minimization problem for linear classifier can   be solved in polynomial time only when the strategic classification problem exhibits certain \emph{adversarial} nature. Specifically, an instance is said to have \emph{adversarial} preferences if all negative test points prefer label $+1$ (but possibly to different extents) and all positive test points prefer label $-1$.  A strictly more relaxed situation has  \emph{essentially adversarial} references --- i.e.,  any negative test point prefers label $+1$ more than any positive test point.
%\footnote{This is a strict generalization of   adversarial preferences since positive points may also prefer label $+1$ in this case, but not so much as negative points.} 
We show that for instance-invariant cost functions, any essentially adversarial instance can be solved in polynomial time whereas for instance-wise cost functions, only  adversarial instances can be solved in polynomial time. These positive results are essentially the best one can hope for. Indeed, we prove that the following situations, which goes slightly beyond the tractable cases above, are both NP-hard: (1) instance-invariant cost functions but general preferences; (2) instance-wise cost functions but essentially adversarial preferences.

\subsection{Related Work and Comparison}\label{subsec:related}
%{\bf Strategic Classification. } 
Most relevant to ours is perhaps the strategic classification model studied by \cite{hardt2016stratclass} and \cite{zhangincentive}, where \citet{hardt2016stratclass} formally formulated the strategic classification problem as a repeated Stackelberg game and \citet{zhangincentive} studied the PAC-learning problem and tightly characterized the sample complexity via ''incentive-aware ERM''. However, their model and results all assume homogeneous agent preferences, i.e., all agents \emph{equally} prefer label $+1$. Our model strictly generalizes the model of \cite{hardt2016stratclass,zhangincentive}  by allowing agents' \emph{heterogeneous} preferences over classification outcomes. Besides the modeling differences, the research questions we study are also quite different from, and not comparable to,  \cite{hardt2016stratclass}. Their positive results are derived under the assumption of \emph{separable cost} functions or its variants. While our characterization of SVC equaling at most $2$ under separable cost functions implies a PAC-learnability result of \cite{hardt2016stratclass}, this characterization serves more as our case study and our main contribution here is the study of the novel concept of SVC, which does not appear in previous works. Moreover, we study the efficient learnability of linear classifiers with cost functions induced by semi-norms. This broad and natural class of cost functions is not separable, and thus the results of \citet{hardt2016stratclass} does not apply to this case.

Our model also   generalizes the setup of  \emph{adversarial classification with evasion attacks}, which has been studied in numerous applications, particularly deep learning models \cite{biggio:ecml13, biggio:icml12,li2014feature,carlini:sp17, goodfellow2015explaining, Jagielski2018ManipulatingML, moosavi:cvpr17, mozaffari:jbhi15, rubinstein:sigmetrics09}; however, most of these works do not yield theoretical guarantees. Our work  extends and strictly generalizes the recent work of \cite{cullina:nips18} through our more general concept of SVC  and results on computational efficiency. In a different work,~\cite{awasthi2019robustness}  studied \emph{computationally} efficient learning of linear classifiers in adversarial classification with $l_{\infty}$-norm-induced $\delta$-ball for allowable adversarial moves. Our computational tractability results generalize  their results to $\delta$-ball  induced by \emph{arbitrary semi-norms}.\footnote{ \cite{awasthi2019robustness} also studied computational tractability of learning other classes of classifiers, e.g., degree-2 polynomial  threshold classifiers, which we do not consider.}  

% The above questions are related to the area of adversarial machine learning, which involves studying vulnerabilities due to manipulations in training or test data, e.g.,~\cite{cullina:nips18, biggio:ecml13, biggio:icml12, carlini:sp17, goodfellow2015explaining, Jagielski2018ManipulatingML, moosavi:cvpr17, mozaffari:jbhi15, rubinstein:sigmetrics09, awasthi2019robustness}. However, strategic learning is more complex in a sense, because each data point is manipulated based on its local utility, unlike in adversarial perturbations, and prior results for the adversarial setting are not easily usable for strategic classification.

% Our result on randomization in strategic classification is related to the result of~\cite{braverman2020role}, who show that a randomized linear classifier can strictly improve strategic classification even in $1$-dimensional feature space, when the cost function is \emph{instance-wise}. Our results complements this finding by proving that when cost function is \emph{instance-invariant}, randomization will not be helpful in $1$-dimensional feature space, but we show it can strictly improve classification accuracy in $2$-dimensional feature space. Moreover, to our knowledge, our result of showing the limit of randomization in adversarial classification is new.         

Strategic classification has been studied in other different settings or domains or for different purposes, including spam filtering  \cite{bruckner2011stackelberg},    classification under incentive-compatibility constraints \cite{zhangincentive},  online learning~\cite{dong:ec18,chen2020learning}, and understanding the social implications~\cite{akyol2016price,milli:fat19,hu:fat19}.
A relevant but quite different line of recent works study \emph{strategic improvements}~\cite{kleinberg2019classifiers,miller2019strategic,ustun2019actionable,bechavod2020causal,pmlr-v119-shavit20a}. 
Finally, going beyond classification, strategic behaviors in machine learning has received significant recent attentions, including in regression problems \cite{Perote2004StrategyproofEF,dekel2010incentive,chen:ec18}, distinguishing distributions \cite{zhang2019distinguishing,zhang2019samples}, and learning for pricing \cite{amin2013learning,mohri2015revenue,vanunts2019optimal}. Due to  space limit,  we refer the curious reader to Appendix \ref{append:related} for detailed discussions.

\section{Model}
\label{sec:prelim}
% \subsection{Model Overview} 
%Let us first give an overview of our model, which builds on and strictly generalizes~\cite{hardt2016stratclass,milli:fat19}.

{\bf Basic Setup. }We consider binary classification, where each data point is  characterized by a tuple $(\x, y, r)$. Like classic classification setups,  $\x \in \X $ is the feature vector and $y \in \{ +1, -1\}$ is its label. The only difference of our setup from classic classification problems is the additional   $r \in \cR \subseteq  \RR$, which is the data point's (positive or negative) preference/reward of being labeled as $+1$. The data point's reward for label $-1$ is, without loss of generality,  normalized to be $0$.  A classifier is a  mapping $h:\X  \to \{ +1, -1\}$. Our model is essentially the same as that of \cite{hardt2016stratclass,miller2019strategic}, except that the $r$ in our model can be any real value from set $R$ whereas the aforementioned works assume $r=1$ for all data points. %  Most previous works in strategic classification have assumed that each data point always equally prefer label $+1$ \cite{hardt2016stratclass,milli:fat19}. Our model relaxes this assumption to allow each data point to have its own reward $r \in \RR$ for being labeled as $+1$ (w.l.o.g., the reward  for label $-1$ is normalized to $0$). 
Notably, we also allow $r$ to be \emph{negative}, which means some data points   prefer to be classified as label $-1$.  This generalization is natural and very useful because  it allows much richer  agent preferences. For instance, it casts the adversarial/robust  classification problem  as a special case of our model as well (see discussions later). Intuitively, the set $R$ captures the richness of agents' preferences. As we will prove, how rich it is will affect both the statistical learnability and computational tractability of the learning problem. 

 \begin{figure}[ht]
    \centering
    \includegraphics[scale=0.35]{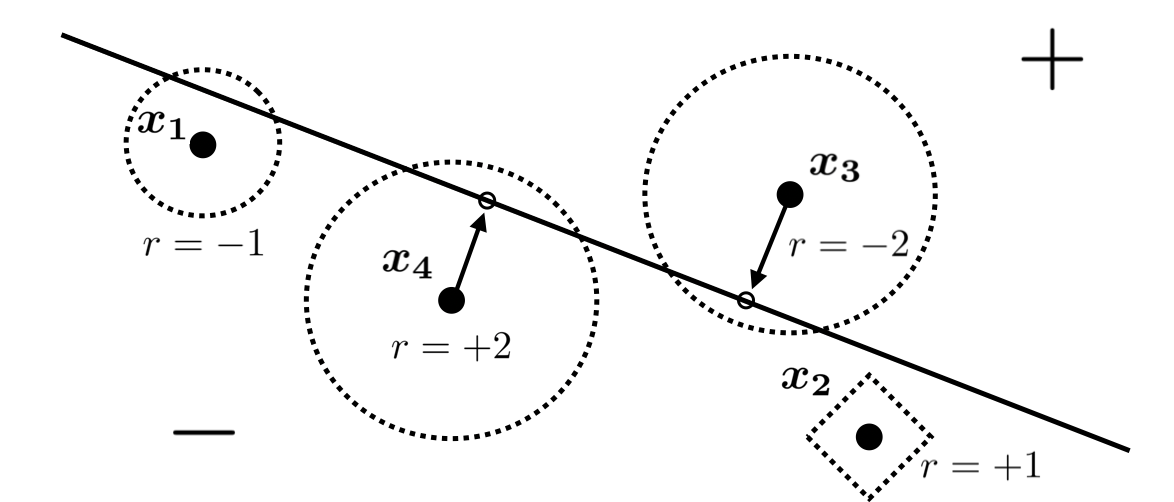}
    \caption{Example illustration of our setup. The line is a linear classifier. Points $\x_3, \x_4$ have incentive to cross the boundary  whereas $\x_1, \x_2$ do not. The dotted cycles contain all manipulated features which have moving cost exactly $1$ and they can be different for different points (i.e., instance-wise costs).}
    \label{fig:ex1}
\end{figure}
\noindent
\textbf{The Strategic Manipulation of \emph{Test} Data.} We consider strategic behaviors during the \emph{test} phase and assume that the training data is unaltered/uncontaminated. An illustration of the setup can be found in Figure \ref{fig:ex1}.  A generic \emph{test} data point is denoted as $(\x, y, r)$.  The  test data point  is strategic and may shift its feature to vector $\z$ with cost $c(\z; \x)$  where $c: \X \times \X  \to \RR_{\geq 0}$. In general,  function $c$ can be an arbitrary non-negative cost function. In our study of strategic linear classification, we assume the cost functions are induced by  \emph{seminorms}. We will consider the following two   types of cost functions, with increasing generality. 
\begin{enumerate}
 \vspace{-2mm} 
   \item  {\bf Instance-invariant cost functions:} A cost function $c$ is \emph{instance-invariant} if there is a common function $l$ such that $c(\z; \x) = l(\z - \x)$ for any point  $(\x, y, r)$.
 \vspace{-2mm}   
  \item {\bf  Instance-wise  cost functions:} A cost function $c$ is  \emph{instance-wise}  if  for each data point $(\x, y, r)$, there is a function $l_{\x}$ such that $c(\z; \x) = l_{\x}(\z - \x)$. Notably, $l_{\x}$ may be different for different data point $(\x, y, r)$.  \vspace{-2mm}  % Since this is the more general situation, we also generically refer it simply as \emph{cost functions} or \emph{general cost functions}.  
\end{enumerate} 
%That is, under instance-invariant cost functions, the manipulation cost of shifting the same amount of feature $\z - \x$ will be the same,  but the cost of the same amount of manipulation will be different under instance-wise cost functions. 
Both cases have been considered in previous works. For instance, the separable cost function studied in \cite{hardt2016stratclass} is instance-wise, and the cost function induced by a seminorm  as assumed by the main   theorem  of \cite{cullina:nips18} is instance-invariant. We shall prove later that the choice among these two types of cost functions will largely affect the efficient learnability of the problem.

Given a classifier $h$, the strategic test data point $(\x,  y, r)$ may shift its feature vector to $\z$ and would like to pick the best such $\z$ by solving the following optimization problem:  
%\anil{(1) either $c(\z; \x)$ or $c(\x, \z)$, (2) we are denoting the data points as $(\x, y, r)$. So we need to think of $h(\z)=y$ or not, (3) $\II(\cdot)\in\{0, 1\}$, right? might need to modify so it is in $\{-1, +1\}$} \hf{(1) let's use $c(\z; \x)$ throughout; (2) it should be $h(\z)=1$ since $r$ is the preference of label +1 over -1, not the preference of true label $y$; (3) $\II(\cdot) $ is correct since the preference of label -1 is normalized to 0, matching the equation $\II(h(\mathbf{z}) = -1) = 0$.  }
\begin{equation}\label{def:bestResponse}
\begin{aligned}
& \text{Data Point Best Response: } \\ 
& \Delta_c(\x,r;h) = \arg \max_{\mathbf{z} \in \X} \big[  \II(h(\z) = 1) \cdot r - c(\z;\x) \big].
\end{aligned}
 \vspace{-2mm}
\end{equation}
where $\II(S)$ is the indicator function and $[  \II(h(\z) = 1) \cdot r - c(\z;\x) ]$ is the quasi-linear \emph{utility function} of data point $(\x,y, r)$. We call $\Delta_c(\x, r;h)$  the  \emph{manipulated feature}. When there are multiple best responses, we assume the data point  may choose any best response and thus will adopt the standard  \emph{worst-case} analysis.   %  We refer to the setting as the \emph{strategic classification} problem,  and the classifiers as \emph{strategic classifiers}.  We emphasize 
Note that the test data's strategic behaviors do  \emph{not} change its true label.  Such strategic  \emph{gaming} behaviors differs from  strategic \emph{improvements}  (see \citeyearpar{miller2019strategic} for more discussions on their differences).

  \vspace{-1mm}
\subsection{The Strategic Classification (\probdet) Problem} 
%Our goal is to efficiently learn a classifier $h$ that have access to only the manipulated feature at deployment (since we don't observe the true features) and minimize expected \emph{0-1 loss} with respect to the underlying true distribution $\D$ of $(\x, y, r)$. 
A \probdet{} problem is described by a  hypothesis class $\H$,  the set of preferences $R$ and a manipulation cost function $c$. We thus denote it as \prob{}. Adopting the standard statistical learning framework,  the input to our learning task are $n$ \emph{uncontaminated} \emph{training} data points $(\x_1, y_1, r_1), \cdots, (\x_n, y_n, r_n)$ drawn independently and identically (i.i.d.) from distribution $\D$.  Given these training data, we look to learn a classifier $h\in \H$ which  minimizes the basic 0-1 loss,   defined as follows:
\begin{equation}\label{eq:def:straExpLoss}
\begin{aligned}
   & \text{Strategic 0-1 Loss of classifier }h: \\ & L_c(h;\D) = \Pr_{(\x,y, r) \sim \D} \big[ h(\Delta_c(\x,r;h)) \not = y) \big]. 
\end{aligned}
 \vspace{-2mm}
\end{equation}
Notably, the classifier $h$ in the above loss definition takes the manipulated feature $\Delta_c(\x,r;h))$ as input and, nevertheless, looks to correctly predict the true label $y$.  
% Though the \probdet{} problem can be studied under any class of hypothesis classes with any loss functions, a big portion of this paper will concern perhaps the most widely studied and used  hypothesis class, i.e., linear classifiers. 
%For some of our results (on efficient learnability), we consider the separable case, where we assume there exists a perfect classifier, deterministic or randomized, with zero loss.  \hf{this sentence might not be needed any more.} 
For notational convenience, we sometimes omit $c$ when it is clear from the context and simply write $\Delta (\x,r;h)$ and $L(h;\D)$.

 \vspace{-1mm}
\subsection{Notable Special Cases }\label{sec:prelim:cases}
Our strategic classification model  generalizes several models studied in previous literature, which we now sketch.\\ %\vspace{-4mm}
%\begin{enumerate}
    %\item  
{\bf Non-strategic classification. } When $\cR = \{ 0 \}$  and $c(\z;\x)>0$ for  any $\x \not =\z$, our model degenerates to the standard non-strategic setting.\\
%\vspace{-1mm}
%\item 
{\bf  Strategic classification with homogeneous preference.}  When $\cR = \{ 1 \}$, our model degenerates to the strategic classification model  studied in prior work \cite{hardt2016stratclass,hu:fat19,milli:fat19}---here all data points have the same incentive of being classified as $+1$.\\ 
%\vspace{-1mm}
%\item 
{\bf Adversarial Classification. } When $R = \{ 1, -1 \}$  (or $\{\delta, -\delta\}$,$\delta\neq 0$), our model becomes the adversarial classification problem \citeyearpar{cullina:nips18,awasthi2019robustness}, where each data point can adversarially move to induce the \emph{opposite} of its true label ---  within the ball of radius $1$ induced by cost function $c$. Our Proposition \ref{prop:svc=avc} provides formal evidence for this connection.\\
%  Then, all the $+1$ points would to move to the $-1$ side whereas all the $-1$ points would like to move to the $+1$ side. This is precisely the adversarial classification models considered in previous works, e.g., \cite{cullina:nips18,awasthi2019robustness}, where the each point is willing to move within $\delta$ ball to be against the classifier.
%\vspace{-1mm}
%\item 
{\bf Generalized Adversarial Classification. } An interesting generalization of the above adversarial classification setting is that $r < 0$ for all data points with true label $+1$ and $r>0$  for all data points with true label $-1$. This captures the  situation where each point has different ``power'' (decided by $|r|$) to play against the classifier. To our knowledge, this generalized setting has not been considered before. %\todo{need to check this claim}.  
    Our results yield new efficient statistical learnability and computational tractability for this setting.     
%\end{enumerate}
 % Thus, our strategic classification model with heterogeneous rewards generalizes non-strategic classification, and the prior strategic and adversarial classification models.

\section{VC-Dimension for Strategic Classification}
\label{sec:stratvc}

In this section, we  introduce the notion of \emph{strategic VC-dimension} (SVC) and show that it properly captures the behaviors of a hypothesis class in the strategic setup introduced above. We then show the connection of SVC with previous studies on both strategic and adversarial learning. %  We will see that the formulation of SVC generalizes both the classic VC-dimension as well as the adversarial VC-dimension introduced recently in \cite{cullina:nips18}. 
%A strategic classification problem can be described by the tuple $\langle \cH, R, c  \rangle$, i.e.,  hypothesis class $\cH: \X \to \{+1,-1\}$, the set of preference values $R$ and cost function $c$.  \anil{the previous sentence is not needed, since \probdet{} is already defined} 
Before formally introducing SVC, we first define the shattering coefficients in strategic setups. 

\begin{definition}[Strategic Shattering Coefficients] 
\label{definition:ssc}
The $n$'th shattering coefficient of any strategic classification problem \prob{} is defined as 
\begin{equation}
    \begin{aligned}\label{eq:shattering-coeff}
    & \sigma_n (\cH, \cR, c )= \max_{\bm{(\x,r)} \in \cX^n \times \cR^n}\\ \nonumber 
    &|\{(h(\Delta_c(\x_1, r_1; h )) , \cdots, h(\Delta_c(\x_n, r_n;h )):h\in \cH\}|,
    \end{aligned}
\end{equation}
where $\Delta_c(\x_i, r_i; h)$   defined in Eq. \eqref{def:bestResponse} is a best response of data point $(\x_i,y_i,r_i)$ to classifier $h$ under cost function $c$. %   $\cH \subseteq (\cX \xrightarrow{} \{-1,+1\})$ is a binary hypothesis class, $r_i$ is the strategic preference of the $i$-th sample and is drawn from $\cR \subset \mathbb{R}$, $c \in (\cX \times \cX \xrightarrow{} \mathbb{R})$ is the cost function, and $\Delta_c(\x, r, h)$ represents the best response function of $\x$ depending on $c, r, $ and $h$.
\end{definition}
That is, $\sigma_n (\cH, \cR, c )$ captures the maximum number of classification behaviors/outcomes (among all choices of data points) that classifiers in $\cH$ can possibly induce by using manipulated features as input. Like classic definition of shattering coefficient,  the $\sigma_n (\cH, \cR, c )$ here does not involve the labels of the data points at all. In contrast, in the shattering coefficient definition for adversarial VC-dimension of \cite{cullina:nips18},  the ``$\max$''  is allowed to be over data labels as well. This is an important difference from us.  
Given the definition of the strategic shattering coefficients, the definition of strategic VC-dimension is standard.    

% \begin{definition}
% Let $H$ denote a hypothesis set. We say that a set of $n$ points $X = (\mathbf{x}_1, r_1), \cdots, (\mathbf{x}_n, r_n)$ is \emph{shattered} by $H$ if for every subset $S\subset X$, there exists a labeling $y_1,\ldots,y_n$ for the points of $X$, and a hypothesis $f\in H$, such that: (1) $y_i=+$ for all $i\in S$, $y_i=-$ for all $i\not\in S$, and (2) $f(\Delta(\x_i, r_i)) = y_i$. We say the strategic VC-dimension of $H$ is $d$ if no subset larger than $d$ can be shattered.
% \end{definition}

\begin{definition} \label{definition:svc}
The Strategic VC-dimension (SVC) for  strategic classification problem \prob{} is defined as
\begin{equation}
    \text{SVC}(\cH,\cR, c)=\sup \{n \in \mathbb{N}: \sigma_n(\cH, \cR,c)=2^n \}. 
\end{equation} 
\end{definition}
%Note that the only assumption on $c(\x, \z)$ is that $c(\x, \x)=0$ for any $\x \in \cX$, and the only assumption for $\cR$ is that the support of $\cR$ contains value $0$, just to account for the case that a data point might choose not to be strategic. Note that if $\cR$ degenerates to $\{+1, -1\}$ (i.e., each data point has the same strategic preference towards either positive or negative), the SVC equals the adversarial VC-dimension \cite{cullina:nips18,awasthi2019robustness} (AVC). 
We show that the SVC  defined above correctly characterizes the learnability of any strategic classification problem \prob{}.   We  consider the standard Empirical risk minimization (ERM) paradigm for strategic classification, but  take into account training data's manipulation behaviors. Specifically, given any cost function $c$, any $n$ uncontaminated  training data points $(\mathbf{x}_1, y_1, r_1), \cdots, (\mathbf{x}_n, y_n, r_n)$ drawn independently and identically (i.i.d.) from the same distribution $\D$, 
% \begin{equation}
% \text{Strategic Empirical Risk: }    L_{\D}(h; c)=\mathbb{E}_{\x,\y,\r\sim \D}   \Big[\mathbb{I} \big[ h (\Delta_c(\x_i, r_i; h)) \not = y_i \big]\Big].
% \end{equation} 
the \emph{strategic empirical risk minimization} (SERM) problem computes a classifier $h \in \cH$ that minimizes the empirical strategic 0-1 loss in Eq. \eqref{eq:def:straExpLoss}. Formally, the SERM    for \prob{} is  defined as follows:
\begin{equation}\label{def:erm}
\begin{aligned}
&\text{SERM} : \quad   \text{argmin}_{h \in \cH} \, \,   L_{c}(h, \{ (\x_i,y_i,r_i) \}_{i=1}^n ) \\& = \sum_{i=1}^n \mathbb{I} \big[ h (\Delta_c(\x_i, r_i; h)) \not = y_i \big]
\end{aligned}
\vspace{-2mm}
\end{equation}
where $L_{c}(h, \{ (\x_i,y_i,r_i) \}_{i=1}^n ) $ is the \emph{empirical loss} (compared to the expected loss $L_c(h, \D)$ defined in Equation \eqref{eq:def:straExpLoss}). %Correspondingly, we use the notation  to denote the expected risk of $h$ under the distribution $\D$ with cost function $c$. 
Unlike the standard (non-strategic) ERM problem and similar in spirit to the ''incentive-aware ERM'' in \cite{zhangincentive}, classifiers in the SERM problem take each data point's strategic response $\Delta_c\ (\x_i, r_i; h) $ as input, while not the original feature vector $\x_i$. % In term of the game-theoretic solution concept, the optimal classifier $h \in \cH$ to SERM is a  Stackelberg equilibrium where the data point's distribution is the empirical distribution from $n$ samples. %\anil{The SERM problem is deterministic problem, not a stochastic optimization. So isn't this just Stackelberg, and not Bayesian Stackelberg?}

Given the definition of strategic VC-dimension and the SERM framework, we state the sample complexity result for PAC-learning in our strategic setup:

% Note that the signed distance from a vector $\x$ to the hyperplane $\w \cdot \x + b = 0$ has the following expression in the case of $\ell_p$ norm (see, e.g., \ref{Melachrinoudis1997analytical}): \todo{This needs to be revised to accomodate any cost function, not only $l_p$ loss.}
% \begin{equation}
%  \frac{\w \cdot \x + b }{ \big[  \sum_{i=1}^d |w_i|^{p/(p-1)} \big]^{(p-1)/p}}, \quad  \forall 1 \leq p \leq \infty, 
% \end{equation}
% where $\big[  \sum_{i=1}^d |w_i|^{p/(p-1)} \big]^{(p-1)/p}$ degenerates to $\max_i |w_i|$ when $p = 1$ and  $\sum_{i=1}^d |w_i|$ when $p = \infty$. We use this for formalizing the ERM problem for linear classifiers.

% \anil{the term evasion adversary has not been defined. Its not needed, once $L_{\D}$ is defined, right?} 
\begin{definition}[PAC-Learnability]
In a strategic classification problem \prob{},  the hypothesis class $\cH \subseteq (\cX \xrightarrow{} \{+1,-1\})$ is \emph{Probably Approximately Correctly (PAC) learnable} by an algorithm $\A$ if there is a function $m_{\cH, R, c}:(0,1)^2 \xrightarrow{}\mathbb{N}$ such that $\forall (\delta,\epsilon) \in (0,1)^2$, for any  $n\geq m_{\cH, R, c}(\delta, \epsilon)$ and any distribution $\D$ for $(\x, y, r)$, with at least probability $1-\delta$, we have $L_{c}(h^*, \D) \leq \epsilon$ where $h^*$ is the output of the algorithm $\A$ with $n$ i.i.d. samples from $\D$ as input.  The problem is \emph{agnostic PAC learnable} if  $L_{c}(h^*, \D) - \inf_{h\in\cH} L_{c}(h,\D) \leq \epsilon $.  % \anil{Instead of ``PAC learnable by SERM'' should we say the ``strategic classification instance $(\cH, R, c)$  is strategically PAC learnable''} \hf{Might no big difference? The definition here is a bit restrictive in the sense that it specified that the learning method must be SERM, while not others... }
\end{definition}

\begin{theorem}
\label{theorem:probdet-pac}
Any strategic classification instance \prob{} is agnostic PAC learnable with sample complexity $m_{\cH,R,c}(\delta,\epsilon)\leq C \epsilon^{-2} [d+\log(\frac{1}{\delta})]$ by the SERM in Eq. \eqref{def:erm}, where $d=SVC(\cH, R, c)$ is the strategic VC-dimension and $C$ is an absolute constant. 
\end{theorem}
\begin{proof}[Proof Sketch]
The key idea of the proof is to convert our new strategic learning setup to a standard PAC learning setup, so that the Rademacher complexity argument can be applied. This is done by viewing the preference $r \in R$ as an additional dimension in the augmented feature vector space. Formally, we consider the new binary hypothesis class $\tilde{H}=\{\kappa_c(h): h\in H\}$, where classifier $\kappa_c$ satisfies $\kappa_c(h) : (\x,r) \mapsto h(\Delta_c(\x, r;h)), \forall (\x, r) \in \cX \times R$. It turns out that the SVC defined above captures the VC-dimension for this new hypothesis class $\tilde{H}$. Formally proof can be found in Appendix \ref{sec:append:thm1}.
\end{proof}

Due to space limit, we defer all formal proofs to the appendix. Proof sketches are provided for some of the results.  Next, we illustrate how SVC  connects to previous literature, particularly the two most relevant works by Cullina et al. \cite{cullina:nips18} and Hardt et al. \cite{hardt2016stratclass}. 

\subsection{SVC generalizes Adversarial VC-Dimension (AVC)} 
We show that SVC generalizes the adversarial VC dimension (AVC) introduced by \cite{cullina:nips18}. We give an intuitive description of AVC here, and refer the curious reader  to  Appendix \ref{prop:svc=avc} for its formal definition. At a high level,  AVC captures the behaviors of binary classifiers under \emph{adversarial} manipulations. Such adversarial manipulations are described by a binary nearness relation $\mathcal{B}\subseteq \cX\times\cX $ and  $(\z; \x)  \in \mathcal{B}$ if and only if data point with feature $\x$ can manipulate its feature to $\z$. %An equivalent, and perhaps more intuitive, representation of the binary nearness relation is to describe it as a subset $\mathcal{B}_{\x} = \{\z': \mathcal{B}(\x, \x + \z') = 1\} \subset \X$ which contains all the moves that $\x$ can make.  
Note that there is no direct notion of agents' \emph{utilities} or \emph{costs} in adversarial classification since each data point simply tries to ruin the classifier by moving within the allowed manipulation region (usually an $\delta$-ball around the data point). Nevertheless, our next result shows that AVC with binary nearness relation $\mathcal{B}$ always equals  to SVC as long as the set of strategic manipulations induced by the data points' incentives is the same as $\mathcal{B}$.  To formalize our statement, we need the following consistency definition. 
% To connect binary nearness relation $\B$ with the agent incentives in our model, we introduce the following natural consistency definition. 
 % with $R=\{1, -1\}$ and a cost function which is an indicator of $\mathcal{B}$. %  \anil{motivation for this might be useful to mention}
\begin{definition}\label{def:bc-consistency}
Given any binary relation $\mathcal{B}$ and any cost function $c$, we say $\mathcal{B},c$ are $r$-consistent if $\mathcal{B} = \{(\z;\x): c(\z;\x) \leq r \}$. In this case, we also say $\mathcal{B}$ [resp. $c$] is $r$-consistent with $c$ [resp. $\mathcal{B}$].   
% Conversely, for any cost function  $c$, a $c$-induced binary relation $ \mathcal{B}_c$ is defined as 
% \begin{equation}
%     \mathcal{B}_c(\x,\z) =   \{ (\x,   \z ):  c( \z;\x) \leq 1 \}.  
% \end{equation}
\end{definition}
% \anil{can we not assume $r=1$ without loss of generality? Also we need to assume $\mathcal{B}$ satisfies a transitivity property?}
By definition  any cost function $c$ is $r$-consistent with the natural binary nearness relation it  induces $\mathcal{B}_c  = \{ (\z;\x): c(\z;\x) \leq r \}$.  Conversely, any binary relation $\mathcal{B}$ is $r$-consistent (for any $r>0$) with a natural cost function that is simply an indicator function of $\mathcal{B}$ defined as follows
% \begin{equation}
%     c[\mathcal{B}](\z; \x) = \inf \{ \lambda \in \RR_+: (\x, \x + \frac{\z - \x }{\lambda}  ) \in \mathcal{B} \}. 
% \end{equation}
\begin{equation}
    c_{\mathcal{B}}(\z; \x) = \begin{cases}
\infty, & \text{if } (\z;\x) \in \mathcal{B} \\
0,  & \text{if }  ( \z;\x) \not \in \mathcal{B} 
\end{cases}. 
\end{equation} 
Note that, $\mathcal{B}$ and $c$ may be $r$-consistent for infinitely many different $r$, as shown in the above example with $\mathcal{B}$ and $c_{\mathcal{B}}$.

\begin{proposition} \label{prop:svc=avc}
For any hypothesis class $\cH$ and any binary nearness relation $\mathcal{B}$, let AVC$(\cH, \B)$ denote the adversarial VC-dimension defined in \cite{cullina:nips18}. Suppose $\mathcal{B}$ and $c$ are $r$-consistent for some $r>0$, then we have AVC$(\cH, \mathcal{B})=$SVC$(\cH, \{+r,-r\}, c)$.
\end{proposition}

 As a corollary of Proposition \ref{prop:svc=avc}, we know that SVC is in general larger than or at least equal to AVC when the strategic behaviors it induces include $\mathcal{B}$. This is formalized in the following statement. 
 % shows that the AVC is essentially a special case of SVC with a preference set $R=\{+1,-1\}$. Therefore, 

\begin{corollary}
\label{corr:svc-avcrelation}
Suppose a cost function $c$ is $r$-consistent with binary nearness relation $\mathcal{B}$ and $ \pm r \in R$, then we have  $$SVC(\cH, R, c)   \geq AVC(\cH, \mathcal{B}).$$ 
\end{corollary}

Corollary \ref{corr:svc-avcrelation} illustrates that for any cost function $c$, the  SVC with a rich preference set $R$  is generally no less than the corresponding AVC under the natural binary nearness relation that $c$ induces. One might wonder how large their gap can be. Our next result shows that for a general $R$  the gap between  SVC and AVC  can be \emph{arbitrarily large}  even in natural setups. The intrinsic reason is that a general preference set $R$ will lead to different extents of preferences (i.e., some data points strongly prefers label $1$ whereas some slightly prefers it). Such variety of   preferences gives rise to more  strategic classification outcomes and renders the SVC larger than AVC, and  sometimes significantly larger, as shown in the following proposition.

\begin{proposition}
\label{prop:svc>avc}
%We always have $AVC(\mathcal{H}, c) \leq SVC (\mathcal{H}, \mathcal{R}, c).$ Moreover, t
For any integer $n>0$, there exists a hypothesis class $\cH$ with point classifiers,  an instance-invariant cost function  $c(\z;\x) = l(\z - \x)$ for some \emph{metric} $c$ and preference set $R$ such that  $ SVC (\mathcal{H}, R, c)= n$ but $VC(\mathcal{H}) = AVC(\mathcal{H}, \mathcal{B}_c(r)) = 1$ for any $ r\in R$  where $\mathcal{B}_c(r) = \{ (\x,\z): c(\z;\x) \leq r \}$ is the natural nearness relation induced by $c$ and $r>0$.   
\end{proposition}
% \begin{proof}[Proof Sketch]
% To prove Proposition \ref{prop:svc>avc}, for any positive integer $n$ we construct an instance with universe set $\mathcal{X}=[n]\cup \mathcal{S}$ where  $[n]= \{1,2,\cdots,n\}$ is the set of $n$ elements and $\mathcal{S}$ be the power set of $[n]$, i.e., the set that contains all the subsets of $[n]$. Therefore, the sample space has size $n+2^n$. The hypothesis class $\cH$ is the set of all the point classifiers with points from   $\mathcal{S}$. We carefully design an   instance-invariant metric  cost function which will lead to the desired VC dimension bounds. The concrete construction and its proof are somewhat involved and we defer the details to Appendix \ref{append:prop:svc>avc}. 
% \end{proof}

\subsection{SVC under Separable Cost Functions} Not only restricting the  set $R$ of preference values can reduce the SVC. This subsection shows that restricting to special classes of cost functions can also lead to a small SVC. One special  class of cost functions studied in many previous works  is the \emph{separable cost functions} \cite{hardt2016stratclass, milli:fat19, hu2019disparate}. %\hf{verify whether they indeed all have separable costs, I think the later two might not have....} \fan{Verified. I think they have.}  
Formally, a cost function $c(\z;\x)$ is separable if there exists function $c_1, c_2: \X \to \RR$ such that $c(\z;\x) = \max \big\{ c_2(\z) - c_1(\x), 0 \big\}$. % We shall assume that staying with the same feature has no cost; that is, $c(\x;\x)=0$ for any $\x\in \cX$. 

The following Proposition \ref{prop:svc-separable} shows that when the cost function is separable,  SVC is at most 2 for \emph{any} hypothesis class $\cH$ and  any class of preference set $R$.\footnote{ The model of \cite{hardt2016stratclass} corresponds to the case $R = \{ 1\}$ in our model. For that restricted  situation, the proof of Proposition \ref{prop:svc-separable} can be simplified to prove SVC = 1 when $R = \{ 1\}$. It turns out that arbitrary preference set $R$ only increases the SVC by at most $1$.}  Therefore,  separable cost function essentially reduces any classification problem to a problem in lower dimension. Together with Theorem \ref{theorem:probdet-pac}, Proposition \ref{prop:svc-separable} also recovers the PAC-learnability result of   \cite{hardt2016stratclass} in their strategic-robust learning model (specifically, Theorem 1.8 of \citeyearpar{hardt2016stratclass})  and,  moreover, generalizes their learnability from  homogeneous agent preferences  to the case with \emph{arbitrary} agent preference values.   

\begin{proposition}\label{prop:svc-separable}
For any hypothesis class $\cH$, any preference set $R$ satisfying $0\not \in R$, and  any separable cost function $c(\z;\x)$, we have SVC$(\cH, R, c)\leq 2$. 
\end{proposition}
The assumption $0 \not \in R$ means each agent must  strictly prefer either label $+1$ or $-1$. This assumption is necessary   since if $0 \in R$, SVC will be at least the classic VC dimension of $\cH$ and thus Proposition \ref{prop:svc-separable} cannot hold. We remark that the above SVC upper bound $2$ holds for any hypothesis class $\cH$. This bound $2$ is tight for some classes of hypothesis, e.g., linear classifiers.

\section{Strategic Linear Classification}\label{sec:linear}
%\subsection{Efficient Learning}
This section instantiates our previous general framework in one of the most fundamental special cases, i.e., linear classification.  We will study both  the \emph{statistical} and \emph{computational} efficiency in   \emph{strategic linear classification}.   Naturally,  we will restrict $\X \subseteq \RR^d$ in this section.  Moreover, the cost functions are always assumed to be induced by semi-norms.\footnote{A function $l: \X \to \RR_{\geq 0}$ is a \emph{seminorm} if it satisfies: (1) triangle inequality: $l (\x + \z) \leq l(\x) + l(\z)$  for any $\x, \z \in \X$; and (2) 
homogeneity: $ l(\lambda \x ) =|\lambda| l(\x)$  for any $ \x \in \X, \lambda \in \RR$.} A linear classifier is defined by a hyperplane $  \w \cdot \x + b = 0$; feature vector $\x$ is classified as $+1$ if and only if $\w \cdot \x +b\geq 0$. With slight abuse of notation, we sometimes also call $(\w,b)$ a \emph{hyperplane} or a \emph{linear classifier}.  Let $\mathcal{H}_d$ denote the hypothesis set of all $d$-dimensional linear classifiers. For linear classifier $(\w,b)$,  the data point's best response can be   more explicit expressed as: $$ \Delta_c(x, r;\w,b) = \arg \max_{\z} \big[  \mathbb{I}(\w \cdot \z + b\geq 0) \cdot r - c(\z;\x) \big] .$$

\subsection{Strategic VC-Dimension of Linear Classifiers}
We first study the statistical learnability by examining the strategic VC-dimension (SVC). Our first result is a negative one, showing that SVC can be unbounded in general even for linear classifiers with features in $\RR^2$ (i.e., $\X \subset \RR^2$) and with simple preference set $R = \{ +1, -1\}$.

 \begin{theorem}
\label{thm:svcinfinite}
Consider strategic linear classification \probL{}.  There is an instance-wise cost function $c(\z; \x) = l_{\x}(\z -\x)$ where each $l_{\x}$ is a norm, such that $SVC(\cH_d, \cR, c) = \infty$ even when $\X \subset \RR^2$ and $\cR = \{1\}$.
\end{theorem}
\begin{proof}[Proof Sketch]
We consider a set of data points $\{\x_i\}_{i=1}^n$ in $\mathbb{R}^2$ whose features are close but with cost functions induced by different norms. The cost functions are designed such that each point $\x_i$ is allowed to strategically alter its feature within a carefully designed polygon $G_i$ centered at the origin. Specifically, for any label pattern $\L \in \{+1,-1\}^n$, it has a corresponding node $s_{\L}$ on a unit cycle. The polygon $G_i$ for $\x_i$ is the convex hull of all $s_{\L}$s whose label pattern $\L$ classifies $i$ as $+1$. 
\begin{figure}[ht]
    \centering
    \includegraphics[scale=0.17]{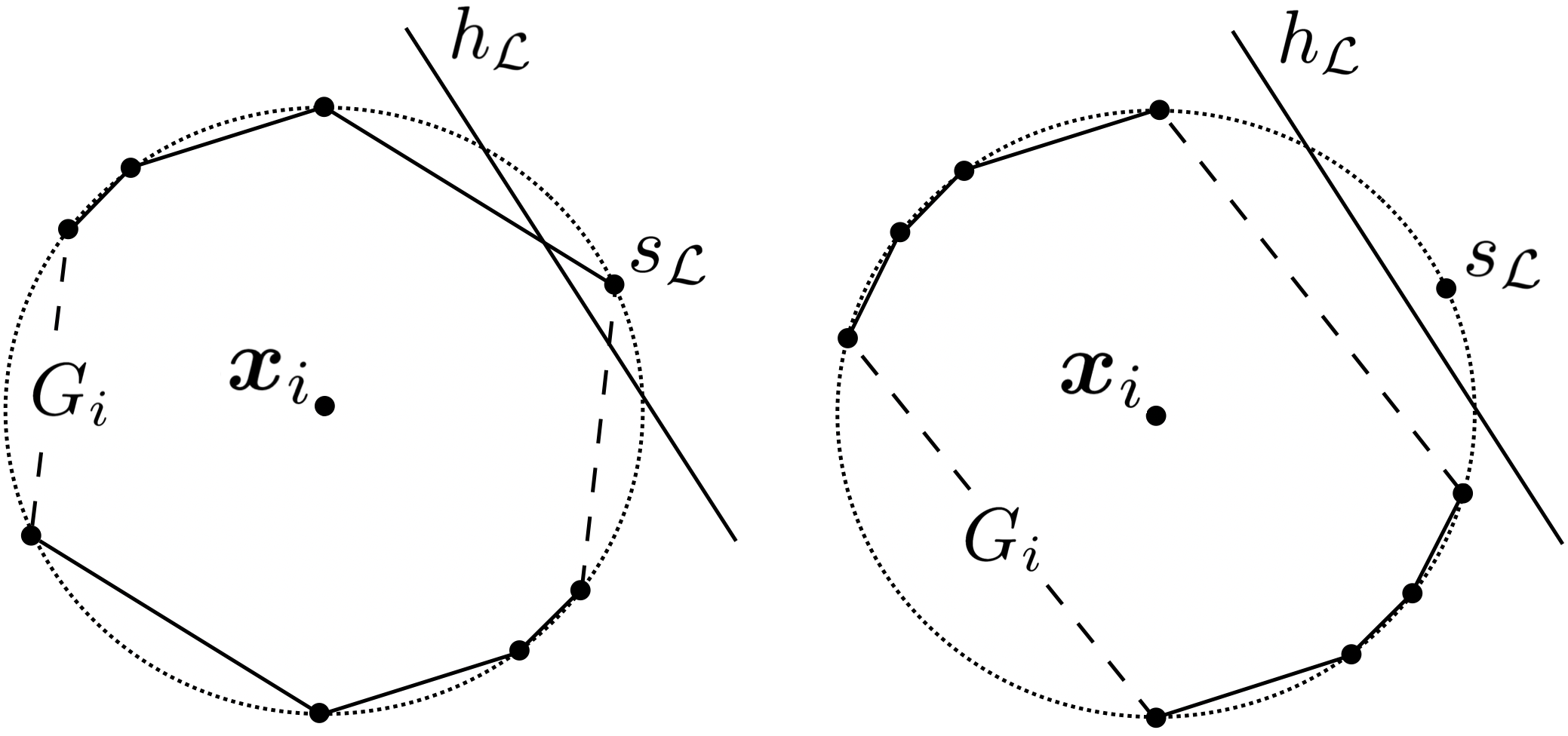}
    \caption{$G_i$ is the  polygon associated with $\x_i$'s cost function. Given any label pattern $\L \in \{+1,-1\}^n$, the linear classifier $h_{\L}$ is carefully chosen to induce exactly the label pattern $\L$ on $\{\x_i\}_{i=1}^n$. % Left: $\x_i$ can fool classifier $h_{\L}$ and will be classified as $+1$. Right: $\x_i$ cannot fool $h_{\L}$ and will be classified as $-1$.
    }
    \label{fig:infsvcExample1}
\end{figure}
Figure \ref{fig:infsvcExample1} illustrates a high-level  idea of our construction. Under such cost functions, we   prove that $\cH$ can induce all $2^n$ possible label patterns. The formal  construction and proof can be found in Appendix \ref{append:svcinfinite}.
%Then by showing for any subset $U \subseteq \{G_i\}_{i=1}^n$, there always exists $h \in \cH$ that intersects and only intersects those polygons in $U$, we prove that $\cH$ can induce all $2^n$ possible label patterns on $\{\x_i\}_{i=1}^n$.  Since $n$ can be arbitrarily chosen, we prove the strategic VC-dimension is infinity. We left the 
\end{proof}

%\anil{The cost function here is not separable (like in the Hardt paper). An interesting question might be if similar results hold for separable function, which would be a contrast with the Hardt results}

In the study of adversarial VC-dimension (AVC)  by \cite{cullina:nips18},  the feature manipulation region of each data point is assumed to be \emph{instance-invariant}. As a corollary, Theorem \eqref{thm:svcinfinite} implies that AVC also becomes  $\infty$ for linear classifiers in $\RR^2$ if each data point's manipulation region is allowed to be different.

It turns out that the $\infty$-large SVC above is mainly due to the instance-wise cost functions. Our next result shows that under instance-invariant cost functions, the SVC will behave nicely and, in fact, equal to the AVC for linear classifiers  despite the much richer data point manipulation behaviors. This result also strictly generalizes the characterization of AVC by \cite{cullina:nips18} for linear classifiers and shows that linear classifiers will be no harder to learn statistically even allowing richer manipulation preferences of data points. 

% Our next few results all rely on the following technical lemma, whose proof is deferred   to Appendix \ref{append:lem:lp-distance}.    
% \begin{lemma}\label{lem:lp-distance}
% For any   seminorm $l:\mathbb{R}^d \xrightarrow{} \mathbb{R}_{\geq 0}$, and the cost function $c(\z;\x)=l(\z-\x)$ induced by $l$, the minimum manipulation cost for $\x$ to move to the hyperplane $\w \cdot \x + b = 0$ is given by the following:
% \begin{equation*}
% \min_{\x'}\{c(\x';\x):\w \cdot \x'+b=0\}= \frac{|\w \cdot \x +b|}{l^*(\w)}
% \end{equation*}
% where $l^*(\w) = \sup_{\z \in B } \{\w \cdot \z \} \in \mathbb{R}_{\geq 0} \cup \{+\infty\}$, and $\B = \{ \z: l(\z) \leq 1 \}$ is the unit ball induced by $l$.
% \end{lemma}

 \begin{theorem}
\label{thm:strat-vc}
Consider strategic linear classification  \probL{}. For any \emph{instance-invariant} cost function $c(\z; \x) = l(\z -\x)$ where   $l$ is a semi-norm, we have SVC$(\cH_d, \cR, c)=d+1-\dim(V_l)$ for any bounded $\cR$, where $V_l$ is the largest linear space contained in the ball $\B=\{\x: l(\x)\leq 1\}$. 

In particular, if $l$ is a norm (i.e.,  $l(\x) = 0$ \emph{iff} $\x=0$), then  $\dim(V_l)=0$ and SVC$(\cH, R, c)=d+1$.  
\end{theorem}
\begin{proof}[Proof Sketch]
%The proof generalizes the proof of Theorem 2 in \cite{cullina:nips18}. We 
The key idea of the proof relies on a careful construction of a linear mapping which, given any set of samples $\{ (\x_i, y_i, r_i) \}_{i=1}^n$,  maps   the class of linear classifiers to  the vector space of points' utilities and moreover, the sign of each data point's utility will correspond exactly to the label of the data point after its strategic manipulation. Using the structure of this construction, we can identify  the relationship between the dimension of the linear mapping and the strategic VC-dimension. By bounding the dimension of the space of signed agent utilities, we are able to derive the SVC though with some involved algebraic argument to exactly pin down the dimension of the linear mapping due to possible degeneracy of $V_l$ ball.  %  Different from \citet{cullina:nips18}'s proof, we need to bound the dimension of the space of signed distances for all possible $\r$ instead of a particular $r$, which takes more involved discussion.
\end{proof}

\subsection{The Complexity of   Strategic Linear Classification} 
In this subsection, we turn our attention to the computational efficiency of learning.  The standard ERM problem for linear classification to minimize the 0-1 loss is already known to be NP-hard in the general agnostic learning setting \cite{feldman2012agnostic}. This implies that agnostic PAC learning by SERM is also NP-hard in our strategic setup. Therefore, our computational study will focus on the more interesting PAC-learning case, that is, assuming there exists a strategic linear classifier that perfectly separates all the data points. In the non-strategic case, the ERM problem can be solved easily by a linear feasibility problem. % The fundamental question we seek to understand is whether the presence of strategic behaviors makes the learning less tractable computationally. 

It turns out that the presence of gaming behaviors does make the resultant SERM problem significantly more challenging. We prove essentially tight computational tractability results in this subsection. Specifically, any strategic linear classification instance can be efficiently PAC-learnable by the SERM   when the  problem exhibits some ``adversarial nature''. However, the SERM problem  immediately becomes NP-hard even when we go  slightly beyond such adversarial situations. We start by defining what we mean by ``adversarial nature'' of the problem.

\begin{definition}[Essentially Adversarial Instances]
For any strategic classification problem \prob{}, let 
\begin{equation}
\begin{aligned}
      &  \text{ min}^- = \min\{r: (\x, y, r)   \text{ with }  y = -1\} \, \,   \text{ and } \\ & \, \,   \text{max}^+ = \max\{r:   (\x, y, r)  \text{ with }  y = +1 \}
\end{aligned}
\end{equation} 
be the minimum reward among all $-1$ points   and the maximum reward among all $+1$ points, respectively. We say the instance is  ``\emph{adversarial}'' if $\min^- \geq 0 \geq  \max^+$ and is ``\emph{essentially adversarial}'' if $\min^- \geq \max^+$. 
\end{definition} 

In other words, an instance is  ``adversarial'' if each data point would like to move to the opposite side of its label though with different magnitudes of preferences, and is ``essentially adversarial'' if any negative data point has a stronger preference to move to the positive side than any positive data point. Many natural settings are essentially adversarial, e.g., all the four examples  in Subsection \ref{sec:prelim:cases}. 

Our first main result of this subsection (Theorem \ref{thm:determistic-easy}) shows that when the strategic classification problem exhibits the above adversarial nature, linear strategic classification can be efficiently PAC-learnable by SERM. The second main result Theorem \ref{thm:determistic-hard} shows that the SERM problem  becomes NP-hard once we go slightly beyond the adversarial setups identified in Theorem \ref{thm:determistic-easy}. These results show that the computational tractability of strategic classification is primarily governed by the preference set $R$. Interestingly, this is in contrast to the statistical learnability results in Theorem \ref{thm:svcinfinite} and \ref{thm:strat-vc} where the preference set $R$ did not play much role. % It is also worthwhile to point out that analogous phenomenon of adversarial games being easier to solve is also observed in classical game theory literature as well. For example, two-player zero-sum bi-matrix games are efficiently solvable by linear programming due to the celebrated minimax theorem by von Neumann, whereas solving two-player general-sum bi-matrix games is computationally intractable as shown in a seminal work by \cite{chen2009settling,}.  
%However, our model and problem are quite different from   bi-matrix games. 
% \begin{definition}[Non-degenerated Asymmetric Seminorm/Seminorm]
% We say $l: \X \to \RR_{\geq 0}$ is an non-degenerated (asymmetric) seminorm if $l(\cdot)$ is an (asymmetric) seminorm and for any $\v \in \cX, \v\neq \bm{0}$, there exists $\lambda>0$ such that $l(\lambda \v)<+\infty$.
% \end{definition}
% Note that a cost function induced by a non-degenerated (asymmetric) seminorm implies that the agent can move towards any direction with finite cost. Now we give the computational efficiency result:

\begin{theorem}\label{thm:determistic-easy}
Any separable strategic linear classification instance \probL{} is efficiently PAC-learnable by the SERM in polynomial time in the following two situations:
\vspace{-4mm}
\begin{enumerate}
    \item The problem is \emph{essentially} adversarial ($\min^-  \geq  \max^+$) and  cost function $c(\z;\x)=l(\z-\x)$ is  instance-invariant and induced by  seminorm $l$. \vspace{-1mm} 
    \item  The problem is  adversarial ($\min^- \geq 0 \geq  \max^+$) and the instance-wise cost function $c(\z;\x) =l_{\x}(\z-\x)$ is induced by seminorms $l_{\x}$. 
\end{enumerate} 
\end{theorem}
\begin{proof}[Proof Sketch]
For situation 1, we can formulate the SERM problem as the following feasibility problem:
\begin{lp*}
\find{  \w, b,   \epsilon >0   }
\sts
\qcon{  \w \cdot \x_i + b   \geq  -r_i     }{  y_i =1 }
\qcon{    \w \cdot \x_i + b  \leq -(r_i + \epsilon)   }{  y_i = -1 }
\con{ l^* (\w) = 1}
\end{lp*}
where $l^*$ is the dual norm of $l$. 

This unfortunately is not a convex program due to the non-convex constraint $l^* (\w) = 1$ --- indeed we prove later that this program is NP-hard to solve in general. Therefore, instead, we solve a relaxation of the above program, by relaxing constraint $l^* (\w) = 1$ to $l^* (\w) \leq 1$ to obtain a convex program. The crucial step of our proof is to show that this relaxation is tight under the \emph{essentially} adversarial assumption. This is proved by converting any optimal solution of the relax convex program to a feasible and optimal solution to the original problem. This is the crucial and also difficult step since the solution to the relaxed convex program may not satisfy $l^* (\w) = 1$ --- in fact, it will satisfy $l^* (\w) < 1$ generally which is why the original program is NP-hard in general. Fortunately, using the essentially adversarial assumption, we are able to employ a carefully crafted construction to generate an optimal solution the the above non-convex program.

For situation 2, we can formulate it as another \emph{non-convex} program with parameter $\w, \epsilon$:  
\begin{lp*} 
\find{  \w, b,   \epsilon >0   }
\sts
\qcon{  \w \cdot \x_i + b   \geq (-r_i) \cdot l^*_{\x_i}(\w)}{ r_i < 0}
\qcon{    -( \w \cdot \x_i + b )  \geq (r_i + \epsilon) \cdot l^*_{\x_i}(\w)  }{  r_i \geq 0 }
\end{lp*}
Fortunately, for any fixed $\epsilon>0$, the above program is convex in variable $\w$.   Moreover, if the system is feasible for   $\epsilon_0>0$ then it is feasible for any $0 < \epsilon \leq \epsilon_0$. This allows us to determine the feasibility  problem above for any $\epsilon$ via binary search, which overall is in polynomial time. 
\end{proof}
Our next result shows that the positive claim in Theorem \eqref{thm:determistic-easy} are essentially the best one can hope for. Indeed, the SERM immediately becomes NP-hard if one goes slightly beyond the two tractable situations in  Theorem \eqref{thm:determistic-easy}. Note that our results did not rule out the possibility of other computationally efficient learning algorithms other than the SERM. We leave this as an intriguing open problem for future works. 
\begin{theorem}\label{thm:determistic-hard}
Suppose the strategic classification problem is linearly separable, then the SERM Problem   for linear classifiers is NP-hard in the following two situations:
\begin{enumerate}
    \item Preferences are arbitrary and the cost function is instant-invariant and induced by the standard $l_2$ norm, i.e., $c(\z; \x) = \|\x - \z||_2^2$. 
     \item The problem is \emph{essentially} adversarial ($\min^-  \geq  \max^+$) and the cost function is  instance-wise and induced by  norms.
\end{enumerate} 
\end{theorem}

\begin{remark}
Theorem \ref{thm:strat-vc}, Theorem \ref{thm:determistic-easy} and Theorem \ref{thm:determistic-hard} together imply that 
for strategic linear classification:\\
(1) the problem is efficiently PAC-learnable (both statistically and computationally)  when the cost function is instance-invariant and preferences are essentially adversarial;\\
(2) SERM can be solved efficiently but SVC is infinitely large when the cost function is instance-wise and   preferences are adversarial;\\
(3) the problem is efficiently PAC learnable  in statistical sense, but SERM is NP-hard when  the cost function is instance-invariant and preferences are arbitrary.    

% \begin{enumerate}[leftmargin=*]
%     \item  the problem is efficiently PAC-learnable (both statistically and computationally)  when the cost function is instance-invariant and preferences are essentially adversarial;
%     \item SERM can be solved efficiently but SVC is infinitely large when the cost function is instance-wise and   preferences are adversarial;
%     \item the problem is efficiently PAC learnable  in statistical sense, but SERM is NP-hard when  the cost function is instance-invariant and preferences are arbitrary.    
% \end{enumerate} 
\end{remark} % \hf{Maybe move later to give an overall summary and implications} 
\section{Summary}
In this work, we propose and study a general strategic classification setting where data points have different preferences over classification outcomes and different manipulation costs. We establish the PAC-learning framework for this strategic learning setting and characterize both the statistical and computational learnability result for linear classifiers. En route, we generalize the recent characterization of adversarial VC-dimension  \cite{cullina:nips18} as well as computational tractability for learning linear classifiers by \cite{awasthi2019robustness}. Our conclusion reveals two important insights. First, the additional intricacy of having different preferences harms the statistical learnability of general hypothesis classes, but \emph{not} for linear classifiers. Second, learning   strategic linear classifiers can be done efficiently only when the setup exhibits some adversarial nature and becomes NP-hard in general. 

% Our learnability result for linear classifiers applies to cost functions induced by semi-norms. However, it is intriguing to ask what the classifier can learn in the presence of adversaries equipped with a broader class of cost functions. So one future direction is to generalize the theory to cost function induced by asymmetric semi-norms or even any metrics. We also note that the strategic classification model we consider is under the full-information assumption. When computing the classifier, we assume the cost function and the strategic preferences are transparent. This is analogous to the evasion attack in the adversarial machine learning literature, where the training data is supposed to be uncontaminated and the manipulation only happens during testing. Multiple interesting directions arise if we try to relax the full-information setting: what if we cannot observe the strategic preferences during training or do not know the adversaries' cost function? We can reformulate the new problem as online learning through repeated Stackelberg games. This strategic classification problem has been studied in \cite{dong:ec18}, but it does not apply to classifiers with 0-1 loss. So it is still interesting to understand the behavior of the optimal classifier in the partial information strategic setting.

Our learnability result for linear classifiers applies to cost functions induced by semi-norms. A future direction is to generalize the theory to cost function induced by asymmetric semi-norms or even any metrics. We also note that the strategic classification model we consider is under the full-information assumption, i.e., the cost function and the strategic preferences are transparent. This is analogous to the evasion attack in the adversarial machine learning literature, where the training data is supposed to be uncontaminated and the manipulation only happens during testing. What if we cannot observe the strategic preferences during training or do not know the adversaries' cost function? This can be reformulated as online learning through repeated Stackelberg games and has been studied in \cite{dong:ec18}, but it does not apply to classifiers with 0-1 loss. 

\vspace{3mm}
{\bf Acknowledgement. } We thank anonymous ICML reviewers for helpful suggestions.  H. Xu is supported by a GIDI award from the UVA Global Infectious Diseases Institute and a Google Faculty Research Award. A. Vullikanti is supported by NSF grants IIS-1931628, CCF-1918656,  NSF IIS-1955797, and NIH grant R01GM109718. R. Sundaram is supported by NSF grants CNS-1718286 and IIS-2039945.

\bibliography{main}
\bibliographystyle{icml2021}
\newpage
\onecolumn

\appendix

\section{Additional discussion on related work}\label{append:related}

 The work of \cite{hardt2016stratclass} is most relevant to our work---they introduced a Stackelberg game framework to model the interaction between the learner and the test data. Our model can be viewed as a  generalization of \cite{hardt2016stratclass} by allowing \emph{heterogeneous} preferences over classification outcomes.  \cite{hardt2016stratclass} assume a special class of  \emph{separably cost functions}, and prove that the optimal classifier is always a \emph{threshold} classifier. Essentially, the assumption of separable cost functions  reduces the feature space to a low dimension, which is also why the strategic VC dimension in this case is at most $2$ as we proved.  Despite this clean characterization, it appears a strong and somewhat unrealistic requirement. For example, one consequence of  separable cost functions is that  for \emph{any} two features $\x, \z$, the manipulation cost from either $\x$ to $\z$ or from $\z$ to $\x$ must be $0$.\footnote{A cost function $c(\z; \x)$ is separable if there exists two functions $c_1, c_2: \X \to \RR$ such that $c(\z; \x) = \max \{ c_2(\z) - c_1(\x), 0 \}$. Since $c(\x; \x) = 0$, we have $c_2(\x) \leq c_1(\x)$ for any $\x$. Therefore, $c_2(\x) + c_2(\z) - c_1(\x) - c_1(\z) \leq 0$. Consequently, either $c_2(\x)  - c_1(\z) \leq 0$ or $ c_2(\z)- c_1(\x) \leq 0$, yielding either $c(\z; \x) =0$ or $c(\x; \z) =0$.} This appears unrealistic in reality. For example, a high-school student  with true average  math grade $80$ and true average literature grade  $95$ is likely to incur cost if she/he wants to appear as $95$ for math and $80$ for literature,  and vice versa. This is  because  different students are good at different aspects. Our model imposes less assumptions on the cost functions. For example, in our study of strategic linear classification, the cost functions are induced by arbitrary semi-norms.  
 %\todo{need to justify this assumptions by citing other works} 
% Working with this broader and more realistic class of cost functions,  we do not have the universal and simple characterization of the optimal classifier any more. This is why we  have to take the classic approach and define learnability for a hypothesis class of classifiers, captured by our novel notion of strategic VC-dimension. 
%As an instantiation of our framework, we then study the efficient learnability, both statistically and computationally, of the ubiquitous linear classifiers. %From this perspective, our insistence on a broader and more realistic class of cost functions necessitate a different set of research questions from \cite{hardt2016stratclass} that we address in this work.  

\cite{bruckner2011stackelberg} is one of the first to consider the Stackelberg game formulation of strategic classification, motivated by spam filtering; however they do not study generalization bounds. \cite{zhangincentive} provide the sample complexity result for strategic PAC-learning under the homogeneous preference setting and in particular study the case under the incentive-compatibility constraints, i.e.,  subject to no data points will misreport features. These two works all assume the positive labels are always and equally preferred. There has also been work on understanding the social implications of strategically robust classification~\cite{akyol2016price,milli:fat19,hu:fat19}; these works show that improving the learner's performance may lead to increased social burden and unfairness. \cite{dong:ec18,chen2020learning} extend strategic linear classification to an online setting where the input features are not known a-priori, but instead are revealed  in an online manner. They both focused on the optimization problem of regret minimization.   Our setting however is in the more canonical PAC-learning setup and our objective is to design statistically and computationally efficient learning algorithms. All these aforementioned works, including the present work,  consider \emph{gaming} behaviors. A relevant but quite different line of recent works study \emph{strategic improvements} where the manipulation does really change the inherent quality and labels \cite{kleinberg2019classifiers,miller2019strategic,ustun2019actionable,bechavod2020causal,pmlr-v119-shavit20a}. The question there is mainly to design  incentive mechanisms to encourage agents' efforts or improvements.  

% Our result on randomization in strategic classification is related to the result of~\cite{braverman2020role}, who show that a randomized linear classifier can strictly improve strategic classification even in $1$-dimensional feature space, when the cost function is \emph{instance-wise}. Our results complements this finding by proving that when cost function is \emph{instance-invariant}, randomization will not be helpful in $1$-dimensional feature space, but we show it can strictly improve classification accuracy in $2$-dimensional feature space. Moreover, to our knowledge, our result of showing the limit of randomization in adversarial classification is new.         

Finally, going beyond classification, strategic behaviors in machine learning has received significant recent attentions, including in regression problems \cite{Perote2004StrategyproofEF,dekel2010incentive,chen:ec18}, distinguising distributions \cite{zhang2019distinguishing,zhang2019samples}, and learning for pricing \cite{amin2013learning,mohri2015revenue,vanunts2019optimal}. These are similar in spirit to us, but study a completely different set of problems using different techniques. Their results are not comparable to ours. 

\section{Omitted Proofs from Section \ref{sec:stratvc}} 
\subsection{Proof of Theorem \ref{theorem:probdet-pac}}\label{sec:append:thm1}
\begin{proof}

Let $\Y=\{+1,-1\}$. Define another binary hypothesis class $\tilde{H}=\{\kappa_c(h): h\in H\}$, where $\kappa_c:(\X \xrightarrow{} \Y) \xrightarrow{} (\X \times R \xrightarrow{}\Y)$  is a mapping such that $\kappa_c(h) (\x, r)=h (\Delta_c(\x, r; h)), \forall (\x, r) \in \cX \times R$. Note that the input of classifier $\kappa_c(h)$ consists of both the feature vector $\x$ and the preference $r$. By the definition of SVC, we have VC$(\tilde{H})=$SVC$(\cH, R, c)=d$.

Given any distribution $\D$, cost function $c$, and $h\in \cH$, the strategic 0-1 loss of $h$ is $L_{c}(h, \D)=\mathbb{E}_{(\x,y,r)\sim \D}\Big[\mathbb{I} \big[ \kappa_c(h) (\x, r) \not = y \big]\Big]=L(\kappa_c(h), \D)$, where $L(\tilde{h}, \D)$ is the standard expected risk of the newly defined $\tilde{h} \in \tilde{\cH}$ under the distribution $\D$ in the non-strategic setting. Therefore, studying the PAC sample complexity upper bound for $\H$ under the strategic setting $\langle R,c \rangle$ is equivalent to  studying the sample complexity for $\tilde{H}$ in the non-strategic setting. The latter problem can be addressed by employing the standard PAC learning analysis. From the Fundamental Theorem of Statistical Learning (Theorem 6.8 in \cite{shalev2014understanding}), we know $\tilde{H}$ is agnostic PAC learnable with sample complexity $O(\epsilon^{-2}(\text{VC}(\tilde{H})+\log\frac{1}{\delta}))$, meaning that there exists a constant $C$ such that for any $(\delta,\epsilon) \in (0,1)^2$ and any distribution $\D$ for $(\x, y, r)$, as long as $n\geq C\cdot\epsilon^{-2}(\text{VC}(\tilde{H})+\log\frac{1}{\delta})$, with at least probability $1-\delta$, we have $$L(\tilde{h}^*, \D) - \inf_{\tilde{h}\in\tilde{H}} L(\tilde{h},\D) \leq \epsilon ,$$ where $\tilde{h}^*$ is the solution of ERM with $n$ i.i.d. samples from $\D$ as input. Let $h^*$ be the solution of the corresponding SERM conditioned on the same $n$ i.i.d. samples from $\D$. By the definition of $\tilde{H}$ and $L_c$, we have $L_{c}(h^*, \D)=L(\tilde{h}^*, \D)$, and $ \inf_{h\in\cH} L_{c}(h,\D)=\inf_{\tilde{h}\in\tilde{H}} L(\tilde{h},\D)$. Therefore, with probability $1-\delta$, we have

\begin{align*}
L_{c}(h^*, \D) - \inf_{h\in\cH} L_{c}(h,\D)\leq \epsilon,
\end{align*}
which implies \prob{} is agnostic PAC learnable with sample complexity $O(\epsilon^{-2} [d+\log(\frac{1}{\delta})])$ by the SERM.

\end{proof}
 
\subsection{Proof of Proposition \ref{prop:svc=avc}}
\begin{proof}
The adversarial VC-dimension defined in   \cite{cullina:nips18} relies on an auxiliary definition of \emph{corrupted classifier} $\tilde{h} = \kappa_R(h)$ of any classifier $h$ for the standard non-adversarial setting such that $\tilde{h}(\x) = h(\x)$ if all the points in $N(\x)$ have the same label as $\x$ and otherwise, $\tilde{h}(\x)=\perp$. Recall that $N(\x)=\{\z\in \cX:(\z;\x) \in \B\}=\{\z\in \cX: c(\z;\x) \leq r \}$ denotes the set of all possible adversarial features $\x$ can move to. Given this auxiliary definition, the  adversarial VC-dimension  is defined as AVC$(\mathcal{H}, \B) = \sup\{n: \sigma_n(\mathcal{F}, \B) = 2^n\}$, where

\begin{equation}\label{eq:sscavc}
  \sigma_n(\mathcal{F}, \B)=\max_{(\x, \y)\in \cX^n\times \{+1,-1\}^n}|\{(f(\x_1, y_1;h),\ldots,f(\x_n, y_n;h)): h\in\mathcal{H}\}|  
\end{equation}
is the shattering coefficient, and  $f(\x_i, y_i) = \mathbb{I}(\tilde{h}(\x_i)\neq y_i)$ is   the \emph{loss function} of the corrupted classifier $\tilde{h} = \kappa_R(h)$.

Since $\B$ and $c$ are $r$-consistent, we have $\mathcal{B} = \{(\z;\x): c(\z;\x) \leq r \}$. Let $R = \{+r, -r\}$.  We now prove the proposition by arguing
\begin{equation}
\label{eq:ssc=ssc}
    \sup \{n \in \mathbb{N}: \sigma_n(\cH, \cR,c)=2^n \}=\sup\{n: \sigma_n(\mathcal{F}, \B) = 2^n\}.
\end{equation}
\begin{enumerate}
    \item If $\sup \{n \in \mathbb{N}: \sigma_n(\cH, \cR,c)=2^n \}=n$, by Definition \ref{definition:ssc}, there exists  $(\x'_i, r'_i) \in \cX \times R, i=1,\cdots,n$ such that $|\{(h(\Delta_c(\x'_1, r'_1; h )) , \cdots, h(\Delta_c(\x'_n, r'_n;h )):h\in \cH\}|=2^n$. Since Definition \ref{definition:ssc} does not rely on the true labels of $\x'_i$, we may let the true labels of $\x'_i$ be $y'_i=-r'_i/r$ for any $i$. In this case, each $\x'_i$'s strategic preference is against its true label, which corresponds to the loss function $f$ in Equation \eqref{eq:sscavc} for the adversarial setting. Therefore, taking $(\x_i, y_i)=(\x'_i, y'_i)$ in Equation \eqref{eq:sscavc} gives $\sigma_n(\mathcal{F}, \B)=2^n$. This implies $\sup \{n \in \mathbb{N}: \sigma_n(\cH, \cR,c)=2^n \} \leq \sup\{n: \sigma_n(\mathcal{F}, \B) = 2^n\}$.
   \item Conversely, if $\sup\{n: \sigma_n(\mathcal{F}, \B) = 2^n\}=n$, from Equation \eqref{eq:sscavc}, there exists $(\x_i, y_i) \in \cX \times R, i=1,\cdots,n$ such that $|\{(f(\x_1, y_1),\ldots,f(\x_n, y_n)): f\in\mathcal{F}\}|=2^n.$ Similarly, taking $r_i=-r y_i \in R$  gives $\sigma_n(\mathcal{H}, R, c)=2^n$, which implies $\sup \{n \in \mathbb{N}: \sigma_n(\cH, \cR,c)=2^n \} \geq \sup\{n: \sigma_n(\mathcal{F}, \B) = 2^n\}$.
\end{enumerate}
Therefore, we have  AVC$(\cH, \mathcal{B})=$SVC$(\cH, \{+r,-r\}, c)$ for any $r$-consistent pair $(\B, c)$.  
\end{proof}

\subsection{Proof of Corollary \ref{corr:svc-avcrelation}}
\begin{proof}
Since $\{+r, -r\}\subseteq \B$, we have $\sigma_n(\cH, R, c)\geq \sigma_n(\cH, \{+r, -r\}, c)$ by Definition \ref{definition:ssc}. As a result, SVC$(\cH, R, c) \geq $SVC$(\cH, \{+r, -r\}, c)$. Then by applying Proposition \ref{prop:svc=avc} we have \\SVC$(\cH, R, c) \geq $SVC$(\cH, \{+r, -r\}, c)=$AVC$(\cH, \mathcal{B})$.
\end{proof}

\subsection{Proof of Proposition \ref{prop:svc>avc} }\label{append:prop:svc>avc}

Given any positive integer $n$, let $[n]$ denotes $\{1,2,\cdots,n\}$, and $\mathcal{S}$ be the power set of $[n]$, i.e., the set that contains all the subsets of $[n]$. Let $\mathcal{X}=[n]\cup \mathcal{S}$ be the sample space of size $n+2^n$, and the hypothesis class $\cH$ is the set of all the point classifiers with points from   $\mathcal{S}$, i.e., $\cH=\{h_{s}: s \in \mathcal{S}\}$, where point classifier $h_{s}$ only classifies the point $s \in \mathcal{S}$ as positive. The cost function $c(z;x) $ is   a metric defined as follows. Since metric is symmetric, i.e., $c(z;x) = c(x; z)$, we will use the notation $c(x,z)$ instead throughout   this proof.  

\begin{equation} \label{eq:costfunc}
c(x,z)=
    \begin{cases}
     x,   & \text{if} \quad x\in [n], z\in \mathcal{S}, x \in z \\ 
     x+1,   &\text{if} \quad x\in [n], z\in \mathcal{S}, x \notin z \\ 
     c(z, x),   &\text{if} \quad x\in \mathcal{S}, z\in [n] \\
     x+z,   & \text{if} \quad x,z\in [n], x\neq z\\
     1, & \text{if} \quad x,z\in \mathcal{S}, x\neq z \\ 
     0, & \text{if} \quad  x=z,
    \end{cases}
\end{equation}
and $R$ is set to be $[-n, -1] \cup [1, n]$. % Since the cost function $c$ is symmetric, we will use the notation $c(x,z)$ instead of $c(z;x)$ throughout the proof.

First, we verify that $c(\cdot, \cdot)$ is indeed  a metric. Given the Definition \eqref{eq:costfunc}, it is easy to see that $c(x,z)=0$ iff $x=z$, and $c(x,z)=c(z,x), \forall x,z \in \cX$. It remains to check the triangle inequality, i.e., for any $x,y,z\in \cX$, $c(x,y)+c(y,z)\geq c(x,z)$. %It's also easy to check the validity of triangle inequality when $x, y, z$ are not distinct, so next we only 
Consider the case when $x, y, z$ are different elements in $\cX$. By enumerating all the possibility that whether each $x,y,z$ is in $[n]$ or $\mathcal{S}$, it suffices to discuss the following $8(=2^3)$ cases: 

\begin{enumerate}
    \item if $x,y,z \in [n]$, $c(x,y)+c(y,z) = x+y+y+z > x+z = c(x,z)$.
    \item if $x,y,z \in \mathcal{S}$, $c(x,y)+c(y,z) = 2 > 1 = c(x,z)$.
    \item if $x,z \in [n], y\in \mathcal{S}$, then $c(x,y)\geq x, c(y,z)\geq z$. $\Longrightarrow$ $c(x,y)+c(y,z) \geq x+z = c(x,z)$.
    \item if $x,y \in [n], z\in \mathcal{S}$, we need to show that $c(x,y)\geq c(x,z)-c(y,z)$. Conditioned on the relationship between $x,y$ and set $z$, the maximum value of $c(x,z)-c(y,z)$ is $x-y+1$ when $y \in z, x\notin z$. Therefore, $c(x,y)=x+y \geq x-y+1 \geq c(x,z)-c(y,z)$.
    \item if $x,z \in \mathcal{S}, y\in [n]$, then $c(x,y)+c(y,z) \geq y+y > 1 \geq c(x,z).$ 
    \item if $x,y \in \mathcal{S}, z\in [n]$, then the maximum value for $c(x,z)-c(y,z)$ is $z+1-z=1$ when $z \notin x, z\in y$. Therefore, $c(x,y)\geq 1 \geq c(x,z)-c(y,z).$
    \item if $x\in\mathcal{S}, y, z \in [n]$, it is equivalent to case 4.
    \item if $y, z \in \mathcal{S}, x\in [n]$, it is equivalent to case 6.
\end{enumerate}

Next, we show VC$(\cH)=1$, AVC$(\cH, \mathcal{B}_c(r) )=1$, and SVC$(\cH, \cR, c)\geq n$. Observe that VC$(\cH)=1$ follows easily since no point classifier $h_s \in \cH$ can generate the label pattern $(+1,+1)$ for any pair of distinct data points.

Next we prove AVC$(\cH, \mathcal{B}_c(r))=1$. We first show AVC$(\cH, \mathcal{B}_c(r))\leq1$ by   arguing that under binary nearness relation $\mathcal{B}_c(r) = \{(z;x): c(z,x) \leq r \}$  with $ r \geq 1$, any two elements $x_1, x_2$ in $\cX$ cannot be shattered by $\cH$.% (i.e., $x$ may prefer to be classified as either $+1$ or $-1$ and can move within the region $N(x)=\{z \in \cX: c(\z;\x)\leq r$\} ). Consider any two elements $x_1, x_2 \in \cX$:

\begin{enumerate}
    \item If at least one of $r_1, r_2$ equals $-r$, e.g., $r_1=-r$, we show that $x_1$ can never be classified as $+1$ by contradiction. Suppose some $h_s \in \cH$ classifies $(x_1, -r)$ as $+1$: if $x_1 \neq s$, since $r_1=-r<0$, $x_1$ will not manipulate its feature and be classified as $-1$; if $x_1 = s$, $x_1$ can move to any $z \in \S$ with cost $1 \leq r$, and will also be classified as $-1$. Therefore, $(x_1, x_2)$ can not be shattered. %\hf{why? explain reasons. This is not obvious. }\fan{explained}
    \item If $r_1=r_2=r$, consider the following two cases:
    \begin{enumerate}
        \item If at least one of $x_1, x_2$ belongs to $\S$, e.g., $x_1 \in \S$, then $x_1$ can move to any $s\in \S$ as $c(x_1, s)=1 \leq r$ for any $s\in \S$. Therefore $x_1$ can never be classified as $-1$ by any point classifier in $\cH$.
        \item if $x_1, x_2 \in [n]$, we may w.l.o.g. assume $x_1<x_2$, i.e., $ x_1+1 \leq x_2$. Observe that when $r<x_1$, any $h_s \in \cH$ will classify $x_1$ as -1 because $c(x_1, s)=x_1 > r, \forall s\in \S$; when $r \geq x_1+1$, any $h_s \in \cH$ will classify $x_1$ as +1 because $c(x_1, s)=x_1+1 \leq r, \forall s\in \S$. Therefore, in order to shatter $(x_1, x_2)$, $r$ must lie in the interval $[x_1, x_1+1) \cap [x_2, x_2+1) = \emptyset$, which draws the contradiction.
    \end{enumerate}
\end{enumerate}

To see that AVC$(\cH, \mathcal{B}_c(r))\geq 1$, for any $x \in [n]$ with $r>0$, it can be classified as either $+1$ or $-1$ as long as $r\in [x,x+1)$. We thus have AVC$(\cH, c)=1$. % \hf{where is this case from? Is it one of the above case? } \fan{the previous argument shows AVC$\leq1$, this case is to give an example showing that AVC can be 1.}

Finally, we prove that SVC$(\cH, R, c) = n$. Consider the subset $[n] \subset \cX$ of size $n$, with each element $i$ equipped with a strategic preference $r_i=i$. For any label pattern $\L \in \{+1, -1\}^n$, let $s_{\L}=\{i \in [n]: \L_i=+1\}$ be an element in $\S$. We claim that $h_{s_{\L}} \in \cH$ gives exactly the label pattern $\L$ on $[n]$. To see this, consider any $i \in [n]$:
\begin{enumerate}
    \item If $i \in s_{\L}$, $i$ will move to $s_{\L} \in \mathcal{S}$ and be classified as $+1$, as the cost $c(i, s_{\L})=i\leq r_i=i$. 
    \item If $i \notin s_{\L}$, $i$ will not move to $s_{\L} \in \mathcal{S}$ and be classified as $-1$, as the cost $c(i, s_{\L})=i+1 > r_i=i$. 
\end{enumerate}

Therefore, any label pattern $\L \in \{+1, -1\}^n$ can be achieved by some $h_{s_{\L}} \in \cH$. This implies SVC$(\cH, R, c) \geq n$. On the other hand, it's easy to see $\cH$ cannot shatter $n+1$ points, because any subset of size $n+1$ must contain an element $s_0$ in $\mathcal{S}$, and no matter what strategic preference $s_0$ has, it will either be classified as $+1$ by all $h_s \in \cH$, or be classified as $+1$ by only one classifier in $\cH$, i.e., $h_{s_0}$. Either case renders the shattering for $n+1$ points impossible.
 
% \newpage 
 
\subsection{Proof of Proposition \ref{prop:svc-separable}}
\begin{proof}
Define the adversarial region for an adversary $(\x, r)$ as $N(\x, r)=\{\z \in \cX:c_2(\z) \leq c_1(\x) + |r|\}\supseteq \{\x\}$. Since staying with the same feature has no cost, this implies $c(\x;\x)=0$ or equivalently $c_2(\x) \leq c_1(\x)$ for any $\x\in \cX$.  Then, the best response function for $(\x, r)$ can be characterized by 
\begin{enumerate}
    \item if $h(\x)=\text{sgn}(r)$, then $h(\Delta(\x, r;h))=\text{sgn}(r)$;
    \item if $h(\x)=-\text{sgn}(r)$, then \begin{equation}\label{eq:fdeltax}h(\Delta(\x, r;h))=\begin{cases}
    -\text{sgn}(r), \quad \forall \z\in N(\x, r): h(\z)=-\text{sgn}(r)\\ \text{sgn}(r),\quad \exists \z\in N(\x, r): h(\z)=\text{sgn}(r)
    \end{cases}\end{equation}
\end{enumerate}

Suppose there are three points $\{(x_i, r_i)\}_{i=1}^3$ that can be shattered by $\cH$. Let $b_i = c_1(\x_i)+r_i$ and w.l.o.g. let $b_1\leq b_2\leq b_3$. From $b_1\leq b_2\leq b_3$, we have $N(\x_1, r_1) \subseteq N(\x_2, r_2)\subseteq N(\x_3, r_3)$. 

By Pigeonhole principle, there must exists two elements in $\{r_1,r_2,r_3\}$ which have the same sign. Suppose these two elements are $r_1, r_2$ and consider the following two cases:

\begin{enumerate}
    \item $r_1>0, r_2>0$. From Equation \ref{eq:fdeltax}, for any $h \in \cH$, $h(\Delta(\x_2, r_2;h))=-1$ means $h(\z)=-1, \forall \z\in N(\x_2, r_2)$. Note that $N(\x_1, r_1) \subseteq N(\x_2, r_2)$, we also have $h(\z)=-1, \forall \z\in N(\x_1, r_1)$. As a result,  $h(\Delta(\x_1, r_1;h))=-1$, meaning the sign pattern $\{+, -\}$ cannot be achieved by any $h\in \cH$ for $\{(\x_1, r_1), (\x_2, r_2)\}$.
    \item $r_1<0, r_2<0$. From Equation \ref{eq:fdeltax}, for any $h \in \cH$, $h(\Delta(\x_2, r_2;h))=1$ means $h(\z)=1, \forall \z\in N(\x_2, r_2)$. Similarly, from $N(\x_1, r_1) \subseteq N(\x_2, r_2)$ we conclude $h(\z)=1, \forall \z\in N(\x_1, r_1)$ and $h(\Delta(\x_1, r_1;h))=1$, meaning the sign pattern $\{-, +\}$ cannot be achieved by any $h\in \cH$ for $\{(\x_1, r_1), (\x_2, r_2)\}$.
\end{enumerate}

Therefore, $\{(x_i, r_i)\}_{i=1}^3$ cannot be shattered by $\cH$, which implies SVC$(\cH, R, c)\leq 2$. 

\end{proof}
 
\section{Proof of Theorem \ref{thm:svcinfinite}}\label{append:svcinfinite}
\begin{proof}
Let $\cX=\mathbb{R}^2$, and consider the  linear hypothesis class on $\cX$: $\cH=\{h=\text{sgn}(\w\cdot \x+b):(\w,b)\in \mathbb{R}^3, \x \in \cX\}$. We show that for any $n \in \mathbb{Z}^+$ and $R=\{+1\}$, there exist $n$ points $\{\x_i\}_{i=1}^n \in \cX^n$ and corresponding cost functions $\{c_i\}_{i=1}^n$, such that the $n$'th shattering coefficients $\sigma_n(\cH, R, \{c_i\}_{i=1}^n)=2^n$ (see Definition \ref{definition:ssc} for $\sigma_n$). Note that the cost function is instance-wise. For convenience, here we equivalently think of it as each data point $i$ has a different cost function $c_i$.

Let $\x_i=(0,0), \forall i\in [n]$ be the set of data points. The main challenge of the proof is a very careful construction of  the cost function for each data point. To do so, we first pick a set of $2^n$ different points  $S=\{s_j\}_{j=1}^{2^n} $  lying on the \emph{unit circle}, i.e., $S \subset \{(x,y): x^2+y^2=1\}$. The number $2^n$ is \emph{not} arbitrarily chosen --- indeed, we will map each point $s_j$ to one of the $2^n$ subsets of $[n]$ in a \emph{bijective} manner so that each $s_j$ corresponds to a unique subset of $[n]$. What are these $2^n$ different points will not matter to our construction neither it matters which point is mapped to which subset so long as it is a bijection.  Moreover, let $\bar{S}=\{(-x, -y): (x, y) \in S\}$ be the set that is origin-symmetric to $S$ such that $\bar{S} \cap S =\emptyset$. $\bar{S}$ is   chosen to ``symmetrize'' our construction to obtain a norm and needs not to have any interpretation.  %Take any bijection $\F: V \mapsto 2^{[n]}$ that maps each point $v_j \in V$ to a \emph{subset} of $[n]=\{1,2,\cdots, n\}$. 
For any $\x_i$, we now define its cost function $c_i$ through the following steps :

\begin{enumerate}
    \item Let $S_i=\{s \subseteq [n]: i \in s\} \subset S$ contains all the $2^{n-1}$ subsets of $[n]$ that include the element $i$. 
    \item Let $\bar{S}_i=\{(-x, -y): (x, y) \in S_i\} \subseteq \bar{S}$ be the set that is origin-symmetric to $S_i$. 
    %Let $T_i=\{\F^{-1}(s): s\in S_i\}$ be the subset of $V$ of size $2^{n-1}$ with each element corresponding to an element in $S_i$ under bijection $\F$, and $\bar{T}_i=\{(-x, -y): (x, y) \in T_i\} \subseteq \bar{V}$ be the set that is origin-symmetric to $T_i$. 
    \item Let $G_i$ be the convex, origin-symmetric polygon with the vertex set being $S_i \cup \bar{S}_i$. 
    \item The cost function of $\x_i$ is defined as $c_i(\z;\x)=\|\x-\z\|_{G_i}$, where $\|\cdot\|_{G_i}=\inf\{\epsilon \in \mathbb{R}_{\geq0}: x\in \epsilon G_i\}$ is a norm derived from polygon  $G_i$ (note the  origin-symmetry of $S_i \cup \bar{S}_i$  and thus $G_i$). 
\end{enumerate}

Next we show that for any label pattern $\L \in \{+1,-1\}^n$, there exists some linear classifier $h \in \cH_2$ such that  $(h(\Delta_{c_1}(\x_1, +1;h ), \cdots, h(\Delta_{c_n}(\x_n, +1;h))=\L$.

With slight abuse of notation, let $s_{\L}= \{i\in[n]: \L_i=+1\}  \in S$ be the point in $S$ that corresponds to the set of the indexes of $\L$ with $\L_i = 1$. Let $h_{\L}$ be any \emph{linear classifier} whose decision boundary intersects the unit circle centered at $\x_i$ and \emph{strictly} separates $s_{\L}$ from all the other elements in $S\cup \bar{S}$. We will use $h_{\L}$ to denote both the linear classifier and its decision boundary (i.e., a line in $\mathbb{R}^2$) interchangeably. Due to the convexity of $G_i$, such $h_{\L}$ must exist. We further let $h_{\L}$ give prediction result $+1$ for the half plane that contains $s_{\L}$ and $-1$ for the other half plane. Figure \ref{fig:infsvcExample} illustrates the geometry of this example.

We now argue that $h_{\L}$ induces the given label pattern $\L$ for instances $\{(\x_i,  1, c_i)\}_{i=1}^n$. To see this, we examine $h_{\L}(\Delta_{c_i}(\x_i,  1;h))$ for each $i$:

\begin{enumerate}
    \item If $i \in  s_{\L} $, then $s_{\L} \in S_i$ and $\x_i$ can move to $s_{\L}$ with cost $c_i(s_{\L}; \x_i) < 1$. This is because $G_i$ is convex and  there exist a point $\x'_i$ on $h_{\L}$ such that $c_i(\x'_i; \x_i) <  c_i(s_{\L};x_i) = 1  = r_i$ (e.g., choose $\x'_i$ as the intersection point of the segment $[\x_i, s_{\L}]$ and $h_{\L}$). Therefore, $h_{\L}$ will classify $\x_i$ as positive. This case is shown in the left panel of Figure \ref{fig:infsvcExample}.
    \item If $i \notin  s_{\L}$, then $s_{\L} \notin S_i$ and $G_i$ does not intersect $h_{\L}$. In this case, $h_{\L}(\x)=-1$, and moving across $h_{\L}$ always induces a cost strictly larger than $1$. Therefore, the best response for $\x_i$ is to stay put and $h_{\L}$ will classify $\x_i$ as negative. This case is shown in the right panel of Figure \ref{fig:infsvcExample}. 
\end{enumerate}

Now we have shown that the $n$'th shattering coefficients $\sigma_n(\cH, \{+1,-1\}, \{c_i\}_{i=1}^n)=2^n$. Since $n$ can take any integer, we conclude the strategic VC-dimension in this case is $+\infty$.

\begin{figure}[ht]
    \centering
    \includegraphics[scale=0.33]{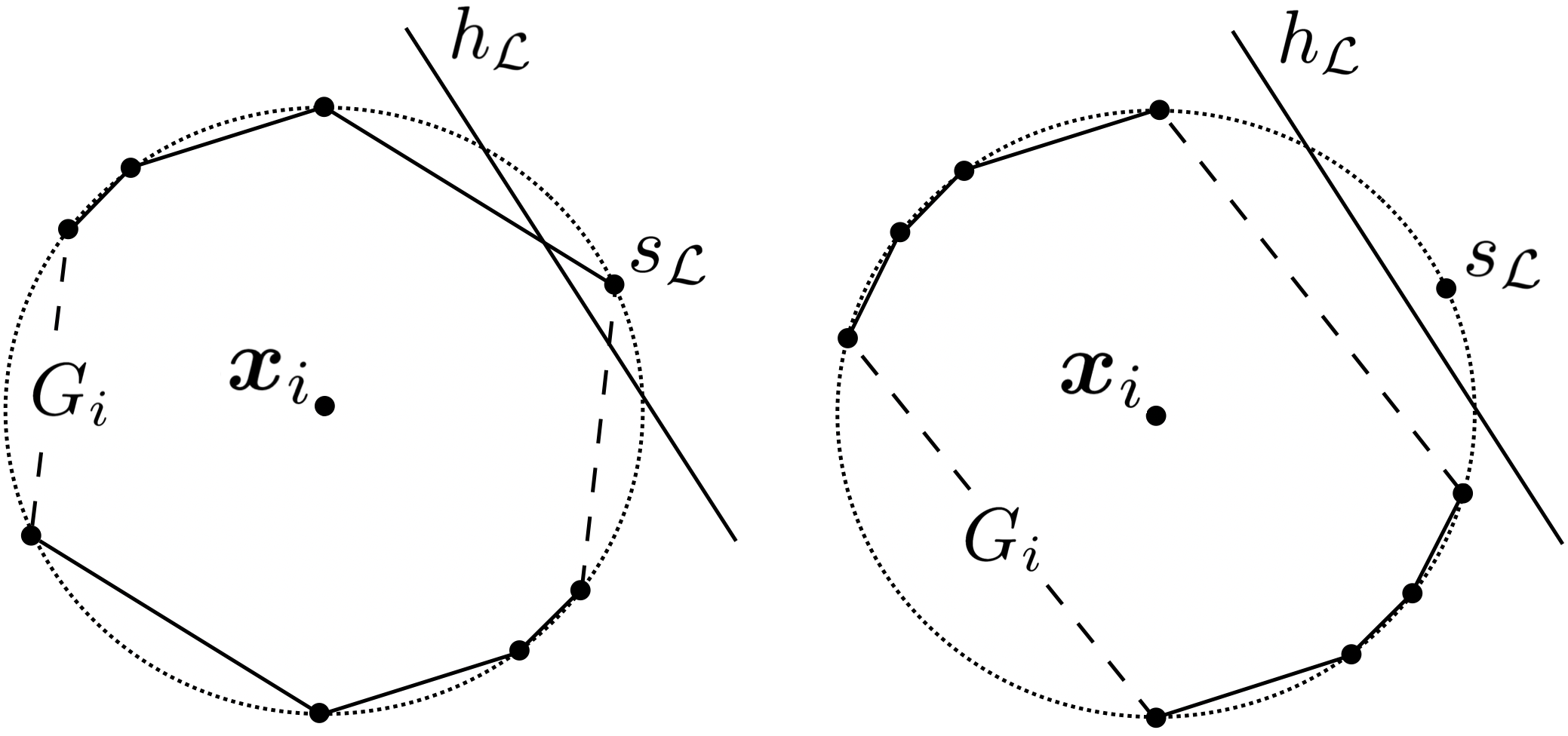}
    % \qquad \qquad
    % \includegraphics[scale=0.33]{figures/example5.1.png}
    \caption{Left: If $i \in s_{\L}$, $h_{\L}$ intersects with $G_i$, and $\x_i$ can manipulate its feature within $G_i$ to cross $h_{\L}$. Right: If $i \notin s_{\L}$, $h_{\L}$ and $G_i$ are disjoint; $\x_i$ cannot manipulate its feature within $G_i$ to cross $h_{\L}$. Given any label pattern $\L \in \{+1,-1\}^n$, $G_i$ is the convex, origin-symmetric polygon associated with $\x_i$'s cost function. The linear classifier $h_{\L}$ is chosen to separate $s_{\L}$ from all other elements in $\bar{S} \cup S$ and classifies $s_{\L}$ as $+1$. The left/right panel shows the two situations, depending on $i \in s_{\L}$ or $i \not \in s_{\L}$. }
    \label{fig:infsvcExample}
\end{figure}

\end{proof}

\section{Proof of Theorem \ref{thm:strat-vc}}
 
The following lemma from advanced linear algebra is widely known  and will be useful for our analysis.  
\begin{lemma}\label{lem:lp-distance}
For any   seminorm $l:\mathbb{R}^d \xrightarrow{} \mathbb{R}_{\geq 0}$, and the cost function $c(\z;\x)=l(\z-\x)$ induced by $l$, the minimum manipulation cost for $\x$ to move to the hyperplane $\w \cdot \x + b = 0$ is given by the following:
\begin{equation*}
\min_{\x'}\{c(\x';\x):\w \cdot \x'+b=0\}= \frac{|\w \cdot \x +b|}{l^*(\w)}
\end{equation*}
where $l^*(\w) = \sup_{\z \in B } \{\w \cdot \z \} \in \mathbb{R}_{\geq 0} \cup \{+\infty\}$, and $\B = \{ \z: l(\z) \leq 1 \}$ is the unit ball induced by $l$.
\end{lemma}

The proof is divided into the following two parts. The first part is the  more involved one.

\vspace{2mm}
\noindent {\bf Proof of SVC$(\cH_d, \cR, c)\leq d+1-\dim(V_l)$:}

It suffices to show that for any $n>d+1-\dim(V_l)$ and $n$ data points $(\x_i,r_i) \in \mathbb{R}^d \times R, \forall i=1,\cdots,n$, there exists a label pattern $ \L \in \{+1,-1\}^n, $ such that for any $h \in \cH_d$ cannot induce $\L $, i.e.,  $$(h(\Delta_c(\x_1, r_1; h), \cdots, h(\Delta_c(\x_n, r_n; h))) \neq \L.$$  

The first step of our proof derives a succinct characterization about the classification outcome for a set of data points. For any seminorm $l$,  it is known the set $\B=\{x:l(x)\leq 1\}$ is nonempty, closed, convex, and origin-symmetric. Let $l^*(\w)=\sup_{\z \in B} \{  \w\cdot \z \}$. We have $l^*(\w)>0$ for all $\w \neq \bm{0}$ since $\bm{0}$ is an interior point of $\B$. According to Lemma \ref{lem:lp-distance}, for any $\x \in \mathbb{R}^d$ and any linear classifier $h=(\w,b) \in \cH_d$, the minimum manipulation cost for $\x$ to move to the decision boundary of $h$ is $|\w\cdot\x+b|/l^*(\w)$. Note that we may w.l.o.g.  restrict to $\w$'s such that $ l^*(\w)=1$ since the sign function $\text{sgn}(\w \cdot \x+b)$ does not change after re-scaling.  For any data point $(\x,r) \in \cX \times R$ and  linear classifier $h \in \cH_d$, we  define the \emph{signed} manipulation cost to the classification boundary as 
\begin{align*}
\delta (h, \x)  =  h(\x) \cdot  \frac{|\w\cdot\x+b|}{l^*(\w)}  = \w\cdot\x+b,
\end{align*}
using the condition $l^*(\w)=1$.  We claim that $h(\Delta_c(\x, r;h))= 2\mathbb{I}( \w\cdot\x+b \geq -r)-1$. This follows a case analysis:
\begin{enumerate}
    \item If $r \leq 0$, then $h(\Delta_c(\x, r, h))=1$  if and only if $h(\x)=1$ and $\x$ \emph{cannot} move across the decision boundary of $h$ within cost $|r| = -r$. This implies $h(\Delta_c(\x,r;h))=2\mathbb{I}( \w\cdot\x+b \geq -r)-1$.
    \item If $r >0$, then $h(\Delta_c(\x, r, h))=-1$ if and only if $h(\x)=-1$ and $\x$ cannot move across the decision boundary of $h$ within cost $r$. In this case, $h(\Delta_c(\x,r;h))=-(2\mathbb{I}(- (\w\cdot\x+b) > r)-1)=2\mathbb{I}( \w\cdot\x+b  \geq -r)-1$. Note that the first inequality holds strictly because we assume $h$ always gives $+1$ for those $\x$ on the decision boundary. 
\end{enumerate} 
For a set of samples $(\bm{X}, \bm{\r})$ where $\bm{X}=(\x_1,\cdots,\x_n), \bm{\r}=(r_1,\cdots,r_n)$, define the set of all possible vectors (over the choice of linear classifiers $(\w, b) \in \H_d$)  of signed manipulation costs as 
\begin{equation}\label{eq:StraVC-linear-map}
\cD(\cH_d, \bm{X})=\{( \w\cdot\x_1+b , \cdots, \w\cdot\x_n+b):h\in \cH_d\},
\end{equation} there is a $h \in \cH_d$ that achieves a label pattern $\L$ on $(\bm{X}, \r)$ if and only if there exist an element in $D(\cH_d, \bm{\x}) + \bm{\r}$ with the corresponding sign pattern $\L$.

 Recall that a linear classifier is described by $(\w,b) \in \mathbb{R}^{d+1}$. The second step of our proof rules out ``trivial'' linear classifiers under strategic behaviors, and consequently allows us to work with only linear classifiers in a   linear space of smaller dimension.  Let $\B=\{\x: l(\x)\leq 1\}$ and $V_l$ be the largest linear space contained in  $\B$. We argue that it suffice to consider only linear classifiers $(\w,b) $ with $\w \perp V_l$. This is because for any $\w$ that is not orthogonal to the subspace $V_l$, we can find $\bar{\z} \in V_l$ such that $c(\bar{\z}; \x) = 0$ and $\w \cdot \bar{\z} \to \infty$ since $V_l$ is a linear subspace. This means any data point can induce its preferred label $\text{sgn}(r)$ with $0$ cost, by moving to  $\bar{\z}$ if $\text{sgn}(r)= +$ and $-\bar{\z}$ otherwise.  Any such linear classifier will result in the same label pattern, simply specified by  $\text{sgn}(r)$.  As a consequence, we only need to focus on linear classifiers $(\w,b) $ with $\w \perp V_l$. Let $\tilde{H_d} = \{ (\w, b): \w \perp V_l \}$ denote all such linear classifiers. 
 
Next, we argue that when restricting to the  non-trivial class of linear classifiers $\tilde{ \cH_d}$,  the $\cD(\tilde{ \cH_d}, \bm{X})$ defined in Equation \eqref{eq:StraVC-linear-map} lies in a linear subspace with dimension at most $d+1 - \text{dim}(V_l)$.  Consider the linear mapping $\cG_{\bm{X}}: \tilde{ \cH_d} \to \mathbb{R}^n$ determined by the data features $\bm{X}$, defined as 
$$\cG_{\bm{X}}(\w, b) = (\w \cdot \x_1+b, \cdots, \w \cdot  \x_n + b), \quad \forall (\w, b) \in \tilde{ \cH_d}.$$ 
Since $\w \perp V_l$, $\w$ is from a linear subspace of $d  - \text{dim}(V_l)$. Linear mapping will not increase the dimension of the image space, therefore $\cD(\tilde{ \cH_d}, \bm{X})$ lies in a space with dimension at most $d+1 - \text{dim}(V_l)$. 

Finally, we prove that there must exist label patterns that cannot be induced by linear classifiers whenever the number of data points $n>d+1-\dim(V_l)$. Let  $\text{span} \big( \cD(\tilde{ \cH_d}, \bm{X}) \big)$ denote the smallest linear space that contains $\cD(\tilde{ \cH_d}, \bm{X}) $. Since $\text{span} \big( \cD(\tilde{ \cH_d}, \bm{X}) \big)$ has dimension at most $d+1 - \text{dim}(V_l)<n$ but $\text{span} \big( \cD(\tilde{ \cH_d}, \bm{X}) \big) \subset \RR^n$, there must exist a non-zero vector $\bar{\u} \in \RR^n$ such that: (1) $ \bar{\u} \not = \bm{0}$; (2) $\bar{\u} \perp \text{span} \big( \cD(\tilde{ \cH_d}, \bm{X}) \big)$ (i.e.,    $\bar{\u} \cdot \bm{v} = 0, \forall \bm{v} \in \text{span} \big( \cD(\tilde{ \cH_d}, \bm{X}) \big)$); and (3)  $\bar{\u} \cdot \bm{\r} \leq 0$ (if $\bar{\u} \cdot \bm{\r} \geq 0$, simply takes its negation). Note that this implies  $\bar{\u} \cdot \bm{v} \leq  0, \forall \bm{v} \in \text{span} \big( \cD(\tilde{ \cH_d}, \bm{X}) \big) + \bm{r}$.

We argue that the sign pattern of the vector $\bar{\u}$, denoted as  $\text{sgn}(\bar{\u})$, and  the sign pattern of all negatives ($\L=(-1,\cdots,-1)$) cannot be achieved simultaneously by $\tilde{\cH_d}$.  Suppose  $\text{sgn}(\bar{\u})$ can be achieved by $\tilde{\cH_d}$, then there must exist $\v^1  \in \text{span}(\cD(\tilde{\cH_d}, \bm X)) +\bm \r $ such that $\text{sgn}(\bar{\u})=\text{sgn}(\v^1 )$ and $\bar{\u} \cdot \v^1 \leq 0$. Since $\text{sgn}(\bar{\u})=\text{sgn}(\v^1)$ also implies $\bar{\u} \cdot \v^1 \geq 0$, we thus have $\bar{\u} \cdot \v^1=\sum_{j=1}\bar{u}_j v^1_j=0$. We claim that there must exist $j$ such that $\bar{u}_j > 0$. First of all, we cannot have $\bar{u}_j <0$ for any $j$ since that implies $v^1_j < 0$ (only strictly less $v^1_j$'s will be assigned $-1$ pattern due to our tie breaking rule) and consequently, $\bar{\u}   \cdot \v^1 < 0$, a contradiction.  Also note that $\bar{\u} \neq \bm 0$, so there exist $j\in [n]$ such that $\bar{u}_j > 0$.  

Utilizing the above property of $\bar{\u}$, we show that the sign pattern $\L=(-1, \cdots, -1)$ cannot be achieved by $\tilde{\cH_d}$. Suppose, for the sake of contradiction, that this is not true. Then there exists another $\v^2 =(v_1^2,\cdots,v_n^2) \in  \text{span} \big( \cD(\tilde{ \cH_d}, \bm{X}) \big) +\bm \r $ with all its elements being strictly negative. Now consider $\v=\v^1-\v^2  \in \text{span}(\cD(\tilde{\cH_d}, \bm X))$, we have $\bar{u} \cdot \v = \bar{u} \cdot \v^1 - \bar{u} \cdot \v^2 = 0- \bar{u} \cdot \v^2>0$.  Here the inequality holds because $\bar{u}_j\geq 0, v^2_j<0$ for all $j$ and there exists some $j$ such that $\bar{u}_j>0$. Therefore, we draw a contradiction to the fact that $\bar{\u} \cdot \v =0$ for any $\v \in \text{span}(\cD(\tilde{\cH_d}, \bm X))$. 

Now we proved that $\text{sgn}(\bar{\u})$ and $\L=(-1,\cdots,-1)$ cannot be achieved simultaneously by non-trivial classifiers $\tilde{\H}_d$, and the only achievable sign pattern for trivial classifiers is $\text{sgn}(\r)$. Note that $\r \in \text{span} \big( \cD(\tilde{ \cH_d}, \bm{X}) \big) + \bm{r}$,  $\text{sgn}(\r)$ is thus also achievable by $\tilde{\H}_d$. Therefore, the trivial classifiers has no contributions to the shattering coefficient, and we conclude at least one of $\text{sgn}(\bar{\u})$ and $\L=(-1,\cdots,-1)$ cannot be achieved by $\H_d$.

\vspace{2mm}
\noindent {\bf Proof of SVC$(\cH_d, \cR, c)\geq d+1-\dim(V_l)$:} 

The second step of the proof shows SVC$(\cH, \cR, c)\geq d+1-\dim(V_l)$ by giving an explicit construction of $(\bm X,\r)$ that can be shattered by $\cH_d$.  Let $\x_0=\bm{0}$, and $(\x_1,\cdots,\x_t)$ be a basis of the subspace orthogonal to $V_l$, $(\x_{t+1}, \cdots,\x_d)$ be a basis of the subspace $V_l$, where $t=d-\dim(V_l)$. 

We claim that the $t+1=d+1-\dim(V_l)$ data points in $\{0,1,\cdots,t\}$ can be shattered by $\cH_d$. In particular, for any given subset $S\subseteq \{0,1,\cdots,t\}$, consider the linear system

\begin{equation*}
    \begin{cases}
     \x_i \cdot \w_S + b_S = 1, \quad \text{if} \quad i \in S\\
     \x_i \cdot \w_S + b_S = -1, \quad \text{if} \quad i \leq t, \text{ and } i \notin S\\
     \x_i \cdot \w_S = 0, \quad t+1\leq i \leq d.
    \end{cases}
\end{equation*}

Because $(\x_1, \cdots,\x_d)$ has full rank, the solution $(\w_S, b_S)$ must exist. Therefore, the half-plane $h=\w_S\cdot \x+b_S$ separates $S$ and $\{\x_0, \cdots,\x_d\}/S$. Now consider the case when each $\x_i$ has a strategic preference $r_i \in \cR$. Since $\w_S$ is chosen to be orthogonal to $V_l$, $\w_S \cdot \x_i$ is bounded when $\x_i \in \{\z:c(\z; \x_i) \leq r_i\}$. Let $\delta_S=\max_{0\leq i\leq t} \{\sup \{\w_S \cdot (\z-\x_i):c(\z; \x_i) \leq r_i\}\}$, and $\delta=\max(1, 2\delta_S)$. Then the data set $\{ \delta\x_0, \cdots,\delta\x_t \}$ can be shattered by $\cH_d$ for any given $c,\cR$, because the classifier $(\delta\w_S, \delta b_S)$ separates the subset $S$ and the other points regardless their strategic responses.

\section{Proof of Theorem \ref{thm:determistic-easy}}
\begin{proof}[Proof of Theorem \ref{thm:determistic-easy}] 

For any data point $(\x, y, r)$, let the manipulation cost for the data point be $c(\z;\x) = l_{\x}(\z - \x)$ where $l_{\x}$ is any  seminorm. Since the instance is separable, there exists a hyperplane $h: \w \cdot \x + b=0$ that separates the given $n$ training points $(\x_1, y_1, r_1), \cdots, (\x_n, y_n, r_n)$ under strategic behaviors. The SERM problem is thus a feasibility problem, which we now formulate. 
Utilizing Lemma \ref{lem:lp-distance} about the signed distance from $\x_i$  to  hyperplane $h$ under cost function $c(\z; \x_i) = l_{\x_i}(\z - \x_i)$, we can  formulate the SERM problem under the separability assumption. Concretely, we would like to find a  hyperplane  $h: \w \cdot \x + b=0$ such that it satisfies the following for any $(\x_i, y_i, r_i)$:
\begin{enumerate}
\item If   $y_i = 1$ and $r_i \geq 0$, we must have either   $\w \cdot \x_i + b \geq 0$ or   $\w \cdot \x_i + b \leq  0$ and $\frac{-(\w \cdot \x + b)}{ l^*_{\x_i}(\w)} \leq  r_i$;
\item If   $y_i = 1$ and $r_i \leq  0$, we must have   $\frac{ \w \cdot \x + b}{ l^*_{\x_i}(\w)} \geq  - r_i$ (this implies   $\w \cdot \x_i + b \geq  0$);   
\item If   $y_i = -1$ and $r_i \leq  0$, we must have either   $\w \cdot \x_i + b \leq 0$ or   $\w \cdot \x_i + b > 0$ and $\frac{\w \cdot \x + b}{ l^*_{\x_i}(\w)} < -r_i$;
\item If   $y_i = -1$ and $r_i \geq  0$, we must have   $\frac{ -( \w \cdot \x + b)}{ l^*_{\x_i}(\w)} >  r_i$ (this implies $\w \cdot \x_i + b <  0$);  
\end{enumerate}
Note that we classify any point on the hyperplane as $+1$ as well, which is  why the strict inequality for Case 3 and 4. Case 1 can be summarized as   $\frac{\w \cdot \x + b}{ l^*_{\x_i}(\w)} \geq - r_i$. Similarly, Case 3 can be summarized as $\frac{\w \cdot \x + b}{ l^*_{\x_i}(\w)} < -r_i$. To impose the strict inequality for Case 3 and 4, we may introduce an $\epsilon$ slack variable. These observations lead to the following formulation of the SERM problem.     

\begin{lp}\label{lp:erm-deterministic} 
\find{  \w, b,   \epsilon >0   } 
\stt
\qcon{  \frac{\w \cdot \x_i + b }{ l^*_{\x_i}(\w)} \geq -r_i}{ \text{ points }(\x_i, y_i,  r_i) \text{ with } y_i=1}
\qcon{  \frac{\w \cdot \x_i + b }{ l^*_{\x_i}(\w)} \leq -r_i-\epsilon}{ \text{ points }(\x_i, y_i,  r_i) \text{ with } y_i=-1 }
\end{lp}

We now consider the two settings as described in the theorem statement. 
We first consider {\bf  Situation 1}, i.e., the essentially adversarial case with $\min^-  \geq \max^+$  and an instance-invariant cost function induced by the same seminorm $l$, i.e., $c(\z; \x) = l(\x - \z)$ for any $\x$.   In this case, System \eqref{lp:erm-deterministic} is equivalent to the following 

\begin{lp}\label{lp:erm-essentialAdv} 
\find{  \w, b,   \epsilon >0   }
\stt
\qcon{  \w \cdot \x_i + b   \geq  -r_i     }{ \text{ points }(\x_i, y_i,  r_i) \text{ with } y_i =1 }
\qcon{    \w \cdot \x_i + b  \leq -(r_i + \epsilon)   }{ \text{ points }(\x_i, y_i,  r_i) \text{ with } y_i = -1 }
\con{ l^* (\w) = 1}
\end{lp}

This system is unfortunately not a convex feasibility problem. To solve System \eqref{lp:erm-essentialAdv}, we consider the following optimization program (OP), which is a relaxation of System \eqref{lp:erm-essentialAdv} by  relaxing  the non-convex constraint $ l^* (\w) = 1$  to  the convex constraint $ l^* (\w) \leq 1$. 
\begin{lp}\label{lp:erm}
\maxi{    \epsilon   }
\stt
\qcon{\w \cdot \x_i + b \geq -r_i  }{ \text{ points }(\x_i, r_i) \text{ with label 1} }
\qcon{ \w \cdot \x_i + b \leq -r_i - \epsilon }{ \text{ points }(\x_i, r_i) \text{ with label -1}}
\con{ l^* (\w)  \leq 1}
\end{lp}

Note that OP \eqref{lp:erm} is a convex program because the objective and constraints are either linear or convex. Therefore, OP \eqref{lp:erm}  can be efficiently solved in polynomial time.\footnote{Note that convex programs can only be solved to be within precision $\epsilon$ in $\text{poly}(1/\epsilon)$ time sine it may have irrational solutions. In this case, we simply say it can be ``solved'' efficiently.  } Note that this relaxation is not tight in general as we will show later that solving System \eqref{lp:erm-essentialAdv}  is NP-hard in general.  

Our main insight is that under the assumption of  $\min^{-} \geq \max^+$, the above relaxation is tight --- i.e.,  there always exists an optimal solution to the above problem with $ l^* (\w) = 1$. This solution is then a feasible solution to System \eqref{lp:erm-essentialAdv} as well,  thus completing our proof. Concretely, given any optimal solution $(\w^*, b^*, \epsilon^*)$ to OP \eqref{lp:erm}, we construct another solution $(\bar{\w}, \bar{b}, \bar{\epsilon})$ as follows: $$ \bar{\w} = \frac{ \w^*}{\alpha} , \quad  \bar{b} = \frac{ b^*}{\alpha} + (\frac{1}{\alpha} -1)\frac{\min^- + \max^+}{2} , \quad \bar{\epsilon} = \frac{ \epsilon^*}{\alpha}, \quad \text{ where } \alpha = l^* (\w^*)  \leq 1.  $$ 
We claim that the constructed solution above  remains feasible to OP \eqref{lp:erm}. Note that for data point with label 1, we have: (1) $ \frac{\min^- + \max^+}{2} \geq r_i$ by assumption $r_i \leq \max^+ \leq \min^-$; (2) $\x_i \cdot \w^* + b^* \geq -r_i  $ by the feasibility of $(\w^*, b^*, \epsilon^*)$. Therefore
\begin{eqnarray*}
&& \x_i \cdot \frac{\w^*}{\alpha} + \frac{b^*}{\alpha}  \geq -\frac{r_i}{\alpha}  \\ 
& \Rightarrow &  \x_i \cdot \frac{\w^*}{\alpha} + \frac{b^*}{\alpha}  + (\frac{1}{\alpha} -1)\frac{\min^- + \max^+}{2}  \geq -\frac{r_i}{\alpha}  + (\frac{1}{\alpha} -1)r_i \\ 
& \Leftrightarrow &  \x_i \cdot \bar{\w} + \bar{b}   \geq -r_i  
 \end{eqnarray*} 
 This proves that the constructed solution is feasible for data points with label 1. Similar argument using the inequality $  \frac{\min^- + \max^+}{2} \leq r_i $ for any negative label data point shows that it is also feasible for negative data points. It is easy to see that the solution quality is as good as the optimal solution $\epsilon^*$ since $\alpha \leq 1$. This proves the optimality of the constructed solution. 

Finally, we consider the {\bf Situation 2 } where the instance is   adversarial, i.e, $\min^- \geq 0 \geq \max^+$. In this case,  $r_i$ in the first constraint of System \eqref{lp:erm-deterministic} is always non-positive whereas $r_i$ in the second constraint is always non-negative. After basic algebraic manipulations, the SERM problem becomes the following optimization problem. 
\begin{lp}\label{lp:erm-generalAdv} 
\find{  \w, b,   \epsilon >0   }
\stt
\qcon{  \w \cdot \x_i + b   \geq (-r_i) \cdot l^*_{\x_i}(\w)}{ \text{ points }(\x_i, y_i,  r_i) \text{ with } r_i \leq  0}
\qcon{    -( \w \cdot \x_i + b )  \geq (r_i + \epsilon) \cdot l^*_{\x_i}(\w)  }{ \text{ points }(\x_i, y_i,  r_i) \text{ with } r_i \geq 0 }
\end{lp}

This is again not a convex feasibility problem due to the non-convex term $(r_i + \epsilon) \cdot l^*_{\x_i}(\w)$, however for any fixed $\epsilon >0$  both constraints are convex. Moreover, if the system is feasible for some $\epsilon_0>0$ and it is feasible for any $0 < \epsilon \leq \epsilon_0$. Therefore, we can determine the feasibility of the (convex) system for any fixed $\epsilon$ and then binary search for the feasible $\epsilon$. Therefore, the feasibility problem in System \eqref{lp:erm-deterministic} can be solved in polynomial time.
\end{proof} 

\section{Proof of Theorem \ref{thm:determistic-hard}}
\begin{proof}

We start with  {\bf Situation 1}, i.e., the preferences are arbitrary but the cost function is $c(\z;\x) = \|\x - \z||_2^2$. We will show later that the second situation can be reduced from the first.  In the first situation, the feasibility problem is System \eqref{lp:erm-essentialAdv} with $l$ as the $l_2$ norm. Our reduction starts by reducing this system to the following optimization problem (OP)

\begin{lp}\label{lp:erm-hardness-1}
\maxi{||\w||_2^2}
\stt
\qcon{\x_i \cdot \w + b \geq -r_i }{ \text{ points }(\x_i, r_i) \text{ with label 1} }
\qcon{ \x_i \cdot \w + b \leq -r_i - \epsilon  }{ \text{ points }(\x_i, r_i) \text{ with label -1}}
\con{||\w||^2_2 \leq 1}
\end{lp}

Formally, we claim that for any fixed $\epsilon$, system \eqref{lp:erm-essentialAdv} is feasible if and only if OP \eqref{lp:erm-hardness-1} has optimal objective value $1$. The ``if" direction is simple. That is, if OP \eqref{lp:erm-hardness-1} has optimal objective value $1$, then the optimal solution $(\w^*, b^*)$ is automatically a feasible solution to System \eqref{lp:erm-essentialAdv} because $||\w^*||_2 = 1$. For the ``only if" direction, let $(\bar{\w}, \bar{b})$ be any feasible solution to System \eqref{lp:erm-essentialAdv}, then it is easy to verify $\w^* = \frac{\bar{\w}}{||\w||_2}$ and $b^* = \frac{\bar{b}}{||\w||_2} $ must also be feasible to System \eqref{lp:erm-essentialAdv}. Moreover, it is an optimal solution to OP \eqref{lp:erm-hardness-1}  with objective value $1$, as desired.

We now prove that determining whether the optimal objective value of OP \eqref{lp:erm-hardness-1} equals $1$ or not is NP-complete. We reduce from the following well-known NP-complete problem called the \emph{partition problem}:
\begin{eqnarray*}
& \text{Given $d$ positive integers $c_1, \cdots, c_d$, decide whether there exists a subset $S \subset [d]$ such that} \\
& \sum_{i \in S} c_i = \sum_{i \not \in S} c_i
\end{eqnarray*}

We now reduce the above partition problem to solving OP \eqref{lp:erm-hardness-1}. Given any instance of partition problem, construct the following SERM instance. 

\vspace{2mm}
\noindent {\bf The Constructed Hard SERM Instance for Situation 1}: We will have $n = 2d + 3$ data points with feature vectors from $\RR^d$. For convenience, we will use $\e_i$ to denote the basis vector in $\RR^d$ whose entries are all $0$ except that the $i$'th  is $1$. For each $i \in [d]$, there is a data point  $(\x, y, r)= (2\sqrt{d} \cdot \e_i, 1, 4)$  as well as a data point $(\sqrt{d} \cdot \e_i, -1, 1-\epsilon)$. The remaining three data points are  $(\c,1, 2)$, data point  $(2\c,-1, 2-\epsilon)$, and data  point  $(3\c,1, 2)$.

We claim that OP \eqref{lp:erm-hardness-1} instantiated with the above constructed instance has an optimal objective value $1$  if and only if the answer to the given partition problem is \emph{Yes}. We first prove the ``if" direction. If the partition problem is a Yes instance, then there exists an $S$ such that $\sum_{i \in S} c_i - \sum_{i \not \in S} c_i = 0$. We argue that the following construction is an optimal solution to OP \eqref{lp:erm-hardness-1} with optimal objective value $1$:
$$ b^* = -2 , \, \,  w_i = \frac{1}{\sqrt{d}} \, \,  \forall i \in S, \, \,  w_i = -\frac{1}{\sqrt{d}} \, \, \forall i \not \in S. $$
Clearly, $||\w^*||_2^2 = 1$. We only need to prove feasibility of $(\w^*, b^*)$. For any label $1$ point  $(\x,r) = (2\sqrt{d} \cdot \e_i, 4)$, we have $ \x \cdot \w^* + b^* =  2\sqrt{d} \e_i \cdot \w^* - 2 = -4  \geq -r $, as desired. Similarly, for any label $-1$  point  $(\x,r) = (\sqrt{d} \cdot \e_i, 1-\epsilon)$, we have $ \x \cdot \w^* + b^* =  \sqrt{d} \e_i \cdot \w^* - 2 = -1 \leq -r-\epsilon$. The feasibility of point $(\c,2)$ with label 1 is argued as follows: $ \x \cdot \w^* + b^* =   \c \cdot \w^* - 2 = -r$. Feasibility of $(2 \c,2-\epsilon)$ and $(3\c, 2)$ are similarly verified. 

We now prove the ``only if'' direction.  In particular, we prove that that if OP \eqref{lp:erm-hardness-1} has some optimal solution $(\w^*, b)$ with $||\w^*||_2^2 =1$, then the partition instance must be Yes. 

Let us first examine the feasibility of OP \eqref{lp:erm-hardness-1}.  
\begin{enumerate}
    \item By the constraints with respect to positive-label data points $(2\sqrt{d} \cdot \e_i, 4)$, we have $2\sqrt{d}  \e_i \cdot \w + b\geq -4$ or equivalently  $w_i  \sqrt{d}  \geq  - \frac{b}{2}-2$.
    \item By the constraints with respect to negative-label data points $(\sqrt{d} \cdot \e_i, 1-\epsilon)$, we have $\sqrt{d}  \e_i \cdot \w + b\leq -1$ or equivalently  $ w_i  \sqrt{d}  \leq -b -1$. 
    \item By the constraints with respect to data point $( \c, 2)$ with label 1, we have $ \c \cdot \w + b\geq -2 $, or equivalently $-2  - b \leq  \c \cdot \w$. 
    \item By the constraints with respect to data point $(2 \c, 2-\epsilon)$ with label -1, we have $2 \c \cdot \w + b\leq -2$, or equivalently  $-2 - b \geq 2 \c \cdot \w$. 
      \item By the constraints with respect to data point $(3 \c, 2)$ with label 1, we have $3 \c \cdot \w + b\geq -2$, or equivalently  $-2 - b \leq 3 \c \cdot \w$. 
\end{enumerate} 
 Point 3--5 implies $ 2 \c \cdot \w \leq  -2  - b \leq  \min \{  \c \cdot \w, 3 \c \cdot \w \}$. This must imply $\c \cdot \w = 0$ as any non-zero  $\c \cdot \w $ cannot satisfy $ 2 \c \cdot \w  \leq  \min \{  \c \cdot \w, 3 \c \cdot \w \}$. As a consequence, the only feasible $b$ value is $b = -2$. Plugging $b=-2$ into Point 1 and 2, we have
 $$ -\frac{1}{\sqrt{d}} \leq w_i \leq \frac{1}{\sqrt{d}}.$$
 Since the optimal objective value is $1 = \sum_{i=1}^d (w_i^*)^2$, it is easy to see that this optimal objective is achieved only when each $w_i^*$ equals either $-\frac{1}{\sqrt{d}} $ or $\frac{1}{\sqrt{d}}$. Now define $S= \{i: w_i^* = \frac{1}{\sqrt{d}}  \} $ to be the set of $i$ such that $w_i^*$ is positive. It is easy to verify that $S$ will be a solution to the partition problem, implying that it is a \emph{Yes} instance. This proves the NP-hardness for {\bf Situation 1} stated in the theorem. 
 
Finally, we consider {\bf Situation 2}  which can be reduced from the first situation. In particular, the constructed hard instance above has reward preferences all being positive (in fact, drawn from only three possible values $\{ 1, 2, 4 \}$), but do not satisfy the essentially adversarial condition. However, if we are allowed to use instant-wise cost functions, we can simply scale down the reward preference for point with label $1$ but propositionally scale down its cost function so that the right-hand-side of the first constraint in System \eqref{lp:erm-essentialAdv} remains the same. Concretely, we now modify our constructed instance above to be the follows. 
 
 \vspace{2mm}
 \noindent {\bf The Constructed Hard SERM Instance for Situation 2}: We still have $n = 2d + 3$ data points with feature vectors from $\RR^d$. For each $i \in [d]$, there is a data point   $(\x, y, r)=(2\sqrt{d} \cdot \e_i, 1, 0.5)$ with  cost function $c(\z;\x) = \frac{1}{8}|\z - \x||_2^2$ as well as a data point   $(\sqrt{d} \cdot \e_i, -1, 1-\epsilon)$ with cost function $c(\z;\x) =  |\z - \x||_2^2$. The remaining three data points are: (1) data point $(\c, 1, 0.5)$ with  cost function $c(\z;\x) = \frac{1}{4}|\z - \x||_2^2$; (2)  data point  $(2\c, -1, 2-\epsilon)$ with cost function $c(\z;\x) =  |\z - \x||_2^2$; (3)  data  point  $(3\c,1, 0.5)$ with cost function $c(\z;\x) = \frac{1}{4}|\z - \x||_2^2$. 
 
It is easy to verify that the above instance satisfy situation 1 in the theorem statement and is equivalent to the instance we constructed for the second situation and thus is also NP-hard.  
\end{proof}

\end{document}